\pdfoutput=1
\documentclass[conference]{IEEEtran}

\usepackage[pagebackref]{hyperref}
\usepackage[numbers,sort&compress]{natbib}
\usepackage{array}
\usepackage{amsmath}
\usepackage{amssymb}
\usepackage{amsthm}
\usepackage{xspace}
\usepackage{dashbox}
\usepackage{nicefrac}
\usepackage{cleveref}
\usepackage{booktabs}
\usepackage[table]{xcolor}
\usepackage{balance}
\usepackage{tikz}
\usepackage{paralist}
\usepackage[algo2e]{algorithm2e}
\usepackage{thmtools,thm-restate}

\usetikzlibrary{graphs,matrix,shapes.misc,calc}
\definecolor{darkblue}{rgb}{0, 0, 0.5}
\hypersetup{colorlinks=true, citecolor=darkblue, linkcolor=darkblue, urlcolor=darkblue}

\renewcommand*{\backref}[1]{}
\renewcommand*{\backrefalt}[4]{%
  \ifcase #1%
    \or Cited on p.~#2.%
  \else Cited on pp.~#2%
  \fi%
}

\newcommand{\MI}{\ensuremath{\mathsf{MI}}\xspace}
\newcommand{\AI}{\ensuremath{\mathsf{AI}}\xspace}
\newcommand{\DPD}{\ensuremath{\mathsf{DPD}}\xspace}
\newcommand{\RC}{\ensuremath{\mathsf{RC}}\xspace}
\newcommand{\PI}{\ensuremath{\mathsf{PI}}\xspace}
\newcommand{\MM}{\ensuremath{\mathsf{MM}}\xspace}
\newcommand{\SMI}{\ensuremath{\mathsf{SMI}}\xspace}
\newcommand{\Inv}{\ensuremath{\mathsf{Inv}}\xspace}

\newcommand{\A}{\ensuremath{\mathcal{A}}\xspace}
\newcommand{\B}{\ensuremath{\mathcal{B}}\xspace}
\newcommand{\Oracle}{\ensuremath{\mathcal{O}}\xspace}
\newcommand{\Train}{\ensuremath{\mathcal{T}}\xspace}

\newcommand{\D}{\ensuremath{\mathcal{D}}\xspace}
\newcommand{\DD}{\boldsymbol{\mathcal{D}}}
\newcommand{\gzero}{\mathcal{G}_0}
\newcommand{\gone}{\mathcal{G}_1}
\newcommand{\bit}{\ensuremath{\{0, 1\}}}
\newcommand{\supp}{\operatorname{supp}}

\newcommand{\nomark}{--}
\newcommand{\xmark}{$\times$}
\newcommand{\boldparagraph}[1]{\noindent\textbf{#1}}
\newcommand{\iid}{i.i.d.\@\xspace}
\newcommand{\eg}{e.g.\@\xspace}
\newcommand{\ie}{i.e.\@\xspace}
\newcommand{\wrt}{w.r.t.\@\xspace}

\newcommand{\Adv}{\mathsf{Adv}}
\newcommand{\reducesTo}[1]{\preceq_{#1}}
\newcommand{\nreducesTo}{\not\preceq}
\newcommand{\Prob}[2]{\Pr\!\left[#1\! :\! #2 \right]}
\newcommand{\guess}[1]{\tilde{#1}}
\newcommand{\Guess}[1]{\widetilde{#1}}

\newcommand\ds{\raisebox{-1pt}{\tikz \draw [line cap=round, line width=0.25ex, dash pattern=on 0pt off 2pt] (0,0) circle [radius=0.75ex];}}

\newcommand{\boxA}[1]{%
    \colorlet{currentcolor}{.}%
    {\color{black}%
    \fbox{\color{currentcolor}#1}}%
}

\newcommand{\boxB}[1]{%
    \colorlet{currentcolor}{.}%
    {\color{red}%
    \setlength{\dashlength}{2pt}\setlength{\dashdash}{1pt}%
    \dbox{\color{currentcolor}#1}}%
}

\newcommand{\boxC}[1]{%
    \colorlet{currentcolor}{.}%
    {\color{blue}%
    \setlength{\dashlength}{6pt}\setlength{\dashdash}{3pt}%
    \dbox{\color{currentcolor}#1}}%
}

\RestyleAlgo{ruled}

\newenvironment{game}[1][htb]
  {
   \begin{algorithm2e}[#1]%
    \DontPrintSemicolon
  }{\end{algorithm2e}
  }

\newenvironment{adversary}[1][htb]
  {
    \begin{algorithm2e}[#1]%
    \DontPrintSemicolon
  }{\end{algorithm2e}
  }

\SetCommentSty{mycommfont}

\newtheorem{definition}{Definition}
\newtheorem{proposition}{Proposition}

\begin{document}

\title{SoK: Let the Privacy Games Begin!\\
A Unified Treatment of Data Inference Privacy in Machine Learning}

\author{
\IEEEauthorblockN{%
Ahmed Salem\IEEEauthorrefmark{1}\IEEEauthorrefmark{3},
Giovanni Cherubin\IEEEauthorrefmark{1},
David Evans\IEEEauthorrefmark{2},
Boris Köpf\IEEEauthorrefmark{1}\\
Andrew Paverd\IEEEauthorrefmark{1},
Anshuman Suri\IEEEauthorrefmark{2},
Shruti Tople\IEEEauthorrefmark{1},
Santiago Zanella-Béguelin\IEEEauthorrefmark{1}\IEEEauthorrefmark{3}}
\IEEEauthorblockA{%
\IEEEauthorrefmark{1}Microsoft\\
\{t-salem.ahmed, giovanni.cherubin, boris.koepf, andrew.paverd, shruti.tople, santiago\}@microsoft.com
}
\IEEEauthorblockA{%
\IEEEauthorrefmark{2}University of Virginia\\
\{evans, as9rw\}@virginia.edu
}
}

\maketitle

\begin{NoHyper}
\def\thefootnote{\IEEEauthorrefmark{3} Corresponding author}\footnotetext{}\def\thefootnote{\arabic{footnote}}
\end{NoHyper}

\begin{abstract}
Deploying machine learning models in production may allow adversaries to infer sensitive information about training data.
There is a vast literature analyzing different types of inference risks, ranging from membership inference to reconstruction attacks.
Inspired by the success of games (\ie probabilistic experiments) to study security properties in cryptography, some authors describe privacy inference risks in machine learning using a similar game-based style.
However, adversary capabilities and goals are often stated in subtly different ways from one presentation to the other, which makes it hard to relate and compose results.
In this paper, we present a game-based framework to systematize the body of knowledge on privacy inference risks in machine learning.
We use this framework to
\begin{inparaenum}[(1)]
  \item provide a unifying structure for definitions of inference risks,
  \item formally establish known relations among definitions, and
  \item to uncover hitherto unknown relations that would have been difficult to spot otherwise.
\end{inparaenum}
\end{abstract}


\section{Introduction}
\label{sec:introduction}
\begin{figure*}[t]
\centering
\vspace{-15mm}
\newcommand{\myangle}{30}
\newcommand{\mybigangle}{75}

\begin{tikzpicture}[node distance={30mm}, thick, main/.style = {draw, circle}]
\node[main] (MI) {MI};
\node[main] (RC) [above right of=MI] {RC};
\node[main] (DPD) [below right of=MI] {DPD};
\node[main] (PI) [above left of=MI] {PI};
\node[main] (AI) [below left of=MI] {AI};

\draw[->] (DPD) edge[bend left=\myangle] node[fill=white, anchor=center, pos=0.5] {\hyperref[thmt@@DPDtoMI]{\ref{thmt@@DPDtoMI}}} (MI);
\draw[->] (MI) edge[bend left=\myangle] node[fill=white, anchor=center, pos=0.6] {\hyperref[thmt@@MItoAI]{\ref{thmt@@MItoAI}}} (AI);
\draw[->] (RC) edge[bend right=\myangle] node[fill=white, anchor=center, pos=0.5] {\hyperref[thmt@@RCtoMI]{\ref{thmt@@RCtoMI}}} (MI);
\draw[->] (DPD) edge[bend right=\myangle] node[fill=white, anchor=center, pos=0.5] {\hyperref[thmt@@DPDtoRC]{\ref{thmt@@DPDtoRC}}} (RC);
\draw[->] (AI) edge[bend left=\myangle] node[fill=white, anchor=center, pos=0.5] {\hyperref[thmt@@AItoMI]{\ref{thmt@@AItoMI}}} (MI);

\draw[->] (PI) edge[bend left=\myangle] node[strike out, draw, -, pos=0.6]{} (MI);
\draw[] (PI) edge[bend left=\myangle] node[fill=white, anchor=center, pos=0.3] {\hyperref[thmt@@PItoMI]{\ref{thmt@@PItoMI}}} (MI);

\draw[->] (MI) edge[bend left=\myangle] node[strike out, draw, -]{} (DPD);
\draw[] (MI) edge[bend left=\myangle] node[fill=white, anchor=center, pos=0.3] {\hyperref[thmt@@MItoDPD]{\ref{thmt@@MItoDPD}}} (DPD);

\draw[->] (MI) edge[bend right=\myangle] node[rotate=90, strike out, draw, -]{} (RC);
\draw[] (RC) edge[bend left=\myangle] node[fill=white, anchor=center, pos=0.3] {\hyperref[thmt@@MItoRC]{\ref{thmt@@MItoRC}}} (MI);

\draw[->] (MI) edge[bend left=\myangle] node[strike out, draw, -]{} (PI);
\draw[] (MI) edge[bend left=\myangle] node[fill=white, anchor=center, pos=0.3] {\hyperref[thmt@@MItoPI]{\ref{thmt@@MItoPI}}} (PI);

\draw[->] (RC) edge[] node[strike out, draw, -]{} (DPD);
\draw[] (RC) edge[] node[fill=white, anchor=center, pos=0.7] {\hyperref[thmt@@RCtoDPD]{\ref{thmt@@RCtoDPD}}} (DPD);

\draw[->] (DPD) edge[bend left=\mybigangle, looseness=2.5] node[strike out, draw, -]{} (PI);
\draw[] (DPD) edge[bend left=\mybigangle, looseness=2.5] node[fill=white, anchor=center, pos=0.6]{\hyperref[thmt@@DPDtoPI]{\ref{thmt@@DPDtoPI}}} (PI);

\draw[->] (AI) edge[red, dashed] node[strike out, draw, -]{} (PI);
\draw[->] (AI) edge[bend right=\myangle, red, dashed] node[strike out, draw, -]{} (DPD);
\draw[->] (PI) edge[bend right=\myangle, red, dashed] node[strike out, draw, -]{} (AI);
\draw[->] (PI) edge[red, dashed] node[strike out, draw, -]{} (RC);
\draw[->] (RC) edge[bend right=\myangle, red, dashed] node[strike out, draw, -] {} (PI);
\draw[->] (PI) edge[bend left=\mybigangle, looseness=2.5, red, dashed] node[strike out, draw, -]{} (DPD);
\draw[->] (AI) edge[bend left=\mybigangle, looseness=2.5, red, dashed] node[rotate=90,strike out, draw, -]{} (RC);

\draw[->] (RC) edge[bend left=\mybigangle, looseness=2.5, red, dashed] (AI);
\draw[->] (DPD) edge[red, dashed] (AI);

\matrix [cells={nodes={anchor=west}}] at ($(current bounding box.east)+(2cm,0)$) {
  \node {MI}; & \node{Membership Inference};\\
  \node {AI}; & \node{Attribute Inference};\\
  \node {DPD}; & \node{DP Distinguishability};\\
  \node {PI}; & \node{Property Inference};\\
  \node {RC}; & \node{Data Reconstruction};\\
  \draw[-latex](0,0) -- ++ (0.6,0); & \node{Reduction};\\
  \draw[-latex](0,0) -- ++ (0.6,0) node[pos=0.2, strike out, draw, -]{}; & \node{Separation};\\
  \draw[-latex,color=red](0,0) -- ++ (0.6,0) [dashed]; & \node{Implied reduction};\\
  \draw[-latex,color=red](0,0) -- ++ (0.6,0) [dashed] node[pos=0.2, strike out, draw, -]{}; & \node{Implied separation};\\
};
\end{tikzpicture}
\vspace{-15mm}
\caption{
Relations among adversary goals (under selected threat models).
A solid arrow from node $A$ to $B$ means that security against $A$ (\ie a nontrivial advantage bound) implies security against $B$.
A struck-through arrow from $A$ to $B$ means that security against $A$ does not imply in general security against $B$; we show this separation with a construction that is secure against $A$ but completely insecure against $B$.
Dashed arrows are implied by solid arrows.
Labels over solid arrows refer to the theorem showing the relationship.
Some separations stem from differences in adversary capabilities, \eg $\MI \not \rightarrow \RC$.
}
\label{fig:relations_graph}
\end{figure*}
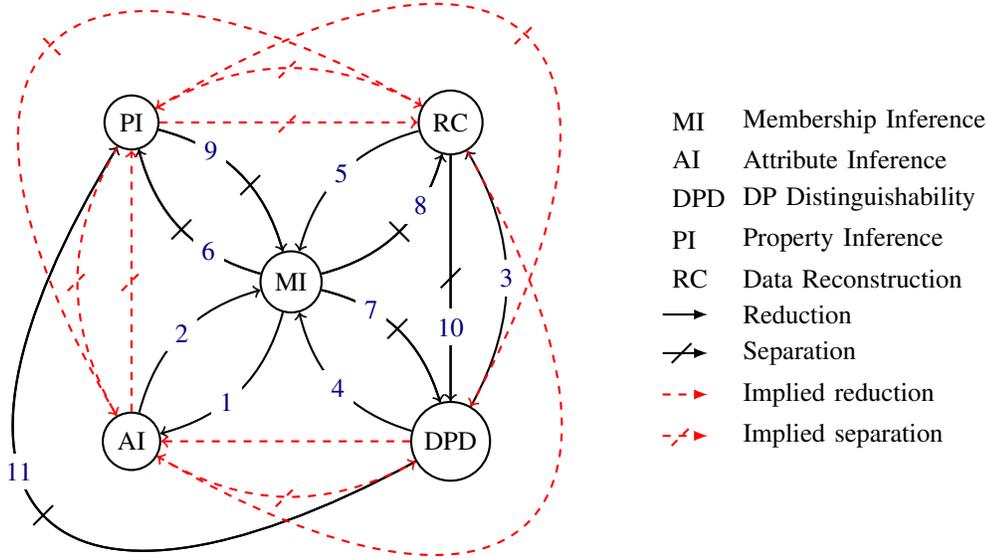

Since the pioneering studies of attribute inference~\cite{Wu:2016,Fredrikson:2015} and membership inference~\cite{Shokri:2017,Li:2013}, research on the inference risks of deploying machine learning (ML) models has bloomed.
There is a growing interest in understanding and mitigating the leakage of information about training data under various threat models that capture different adversarial capabilities (\eg, observing model outputs, model parameters, or transcripts of iterative optimization methods) and goals (\eg, membership inference~\cite{Shokri:2017}, attribute inference~\cite{Wu:2016,Fredrikson:2015}, property inference~\cite{Zhang:2021b,Mahloujifar:2022,Ganju:2018,Suri:2022}, and data reconstruction~\cite{Carlini:2021c,Balle:2022}).

An emerging trend in the literature is to capture threat models using \emph{privacy games}.
This originates from the seminal work of \citet{Wu:2016} on formalizing attribute inference.
A privacy game is a probabilistic experiment where an \emph{adversary} interacts with a \emph{challenger}.
The challenger drives the experiment, invoking the adversary to provide them with information and to allow them to make certain choices, possibly while interacting with oracles controlled by the challenger.
The adversary eventually produces a guess for a confidential value.
This experiment defines a probability space where the success of the adversary can be measured in terms of the probability of their guess being correct.

The use of games for privacy in ML is inspired by the well-established use of games to define and reason about security properties in cryptography.
Cryptographic games are used to standardize and compare security definitions~\cite{Goldwasser:1982,Stern:2003}, and to structure~\cite{Bellare:2006} and even mechanize proofs of security~\cite{Barthe:2009,Blanchet:2012}.
In comparison, the use of privacy games in the ML literature is still in its infancy:
\begin{asparaenum}[(1)]
\item there are no well-established standards for game-based definitions,
\item relationships between different privacy games have only been partially explored, and
\item games are rarely used as an integral part of proofs, despite being especially convenient for this task.
\end{asparaenum}

This has resulted in many game variants in the literature that attempt to formalize the same adversary goal but have subtle yet important differences.
This fragmentation leads to confusion and hinders progress---for membership inference alone, we found variants that differ in details that can change their meaning and substantially alter results.
To address this problem, we present the first systematization of knowledge about privacy inference risks in machine learning, going above and beyond the problem left open since 2016 by \citet{Wu:2016} of merely devising rigorous game-based definitions.
Concretely,
\begin{asparaitem}
\item We break down the \emph{anatomy} of game-based privacy definitions for ML systems into individual components: adversary's capabilities and goals, ways of choosing datasets and challenges, and measures of success (\autoref{sec:anatomy}).

\item Based on this anatomy, we propose a \emph{unified representation} of five fundamental privacy risks as games: membership inference, attribute inference, property inference, differential privacy distinguishability, and data reconstruction (\autoref{sec:formalization}).

\item Using the game-based framework, we \emph{establish and rigorously prove relationships} between the above risks.
Similarly to the study of \emph{concrete security} in cryptography~\cite{Bellare:1997}, we define a quantitative notion of \emph{reduction} between privacy properties.
Using this notion, we prove a set of relations among the above five privacy risks.
This allows us to establish, for every possible ordered pair of risks $A,B$, either a reduction showing that security against $A$ implies security against $B$, or a separation result showing the impossibility of a generic reduction from $A$ to $B$.
\autoref{fig:relations_graph} summarizes the conclusions of this systematization effort for selected games.
\item We present a \emph{case study} (\autoref{sec:case_study}), where we prove that a scenario described as a variant of membership inference in the literature can actually be decomposed into a combination of membership and property inference.
Importantly, in this case we exploit \emph{code-based} reductions, structured as a sequence of games; \ie, our arguments rely on transforming code with a formal semantics.
This way of conducting proofs has seen great success in cryptography.
However, before our work, it had not reached the same level of rigor when reasoning about privacy inference risks in ML.
\end{asparaitem}

\boldparagraph{Scope}
The focus of this SoK is to formalize and systematize game-based definitions that capture the risk of leaking information about the training data of ML models.
We used the following methodology to identify existing game-based definitions from the literature: starting from the seminal works of \citet{Wu:2016} and \citet{Yeom:2020}, we surveyed all peer-reviewed publications in Google Scholar as of August 2022 that cite either of these works. We examined these publications and collected all game-based definitions of attacks that aim to infer information about the training data of ML models.
Our primary objective is to systematize games appearing in the literature.
However, we also demonstrate the versatility of our framework by presenting new game-based definitions of attacks that have not been previously formulated as games.

\boldparagraph{Summary of contributions}
We propose a unifying game-based framework for formalizing privacy inference risks of training data in ML, which we use to systematize definitions from the literature and to establish relations between them.
Our work aims to reduce ambiguity and increase rigor when reasoning and communicating about ML privacy, and gives a solid foundation to future research and decision-making.

\section{Anatomy of a Privacy Game}
\label{sec:anatomy}
Privacy games are parametrized by an adversary (\A) and a training pipeline that specifies the training algorithm (\Train), data distribution (\D), and the size of the training dataset ($n$).
A challenger simulates the ML system. The adversary uses their capabilities---defined by a threat model---to interact with the system and infer information about the training dataset.
\begin{game}
\LinesNumbered
\caption{Membership Inference}
\label{game:MI_example}
  \KwIn{$\A, \Train, n, \D$}
  $S \sim \D^n$		  \tcp*[r]{sample $n$ \iid points from distribution $\D$}
  $b \sim \bit$     \tcp*[r]{flip a fair coin}
  \eIf{$b = 0$}{
	  $z \sim S$  		\tcp*[r]{sample a challenge point uniformly from $S$}
  }
  {
	  $z \sim \D$     \tcp*[r]{sample a challenge point from $\D$}
  }
  $\theta \gets \Train(S)$ \tcp*[r]{train a model $\theta$}
  $\guess{b} \gets \A(\Train, \D, n, \theta, z)$ \tcp*[r]{adversary guesses $b = \guess{b}$}
\end{game}

\autoref{game:MI_example} formalizes the membership inference experiment of \citet{Yeom:2020}, which we use as a running example.
The challenger samples a training dataset $S$ (line 1) and flips a fair coin $b$ (line 2).
Depending on the outcome, they either sample a challenge point $z$ from the training dataset $S$, or from the data distribution \D (lines 3--7).
We discuss alternatives for choosing training datasets and challenges in \autoref{sec:anatomy_choosing}.
The challenger then trains a target model $\theta$ (line 8), and asks the adversary to make a guess $\guess{b}$ for $b$ (line 9).
In this game, the adversary is given the training algorithm (\Train), data distribution (\D), dataset size ($n$), target model ($\theta$), and the challenge point ($z$).
We discuss alternatives for adversary's capabilities in \autoref{sec:anatomy_choosing} and \autoref{sec:anatomy_access}.
The success of the adversary in making a correct guess ($\guess{b} = b$) is measured with respect to the baseline of a random guess. Any advantage over this baseline indicates leakage of membership information.
We discuss other ways to quantify the adversary's success in \autoref{sec:anatomy_evaluating}.

We now discuss in more detail the building blocks of games described above and highlight common choices.

\subsection{Adversary Goals}
\label{sec:anatomy_goals}

We identify five adversary goals from the literature that enable an adversary to directly infer information about the training dataset of an ML model.
We describe these goals informally below and formalize them as games in \autoref{sec:formalization}.

\boldparagraph{Membership Inference (MI)}
The adversary aims to determine whether a specific \emph{record}~\cite{Yeom:2020,Shokri:2017} or \emph{subject}~\cite{Mahloujifar:2021,Suri:2022b} (an entity who may contribute more than one record) was present in the training dataset of the target model.
For example, a successful MI attack against a model trained on clinical records of patients with an infective disease can reveal that a target patient was infected.

\boldparagraph{Attribute Inference (AI)}
The adversary aims to use the model to infer unknown attributes of a record in the training dataset given partial information about the record~\cite{Yeom:2020}.
A successful AI attack can result in the reconstruction of sensitive attributes of a target individual.

\boldparagraph{Property Inference (PI)}
The adversary aims to learn sensitive \emph{statistical} properties of the target model's training distribution.
For example, in a malware classifier, the training dataset may have been generated using a particular testing environment, and it may benefit the adversary to learn certain properties of this environment~\cite{Ganju:2018}.
From an auditing perspective, property inference could be used to assess the training dataset for harms (\eg, under-representation)~\cite{Zhang:2021b}.

\boldparagraph{Differential Privacy Distinguishability (DPD)}
The adversary aims to determine which of a pair of adjacent datasets (\eg differing in the data of one record) of their choosing was used to train the target model.
This goal recasts differential privacy in a game-based setting by making the adversary explicit.
This connection can be used to estimate the differential privacy budget of training pipelines~\cite{Malek:2021,Carlini:2021b,Zanella-Beguelin:2022}.

\boldparagraph{Data Reconstruction (RC)}
The adversary aims to reconstruct samples from the training dataset of a target model~\cite{Carlini:2019,Carlini:2021c,Balle:2022}.
A successful attack can partially reconstruct the training dataset, potentially violating confidentiality requirements.

\boldparagraph{Beyond training data inference}
Other adversary goals, such as model stealing~\cite{Tramer:2016,Orekondy:2019} and hyperparameter stealing~\cite{Wang:2018} are beyond the scope of this SoK because they do not enable the adversary to directly infer information about the training data.
However, the \emph{effects} of these other goals are readily captured by our game-based analysis.
For example, a successful model stealing attack that is used as a precursor to membership inference can be represented by changing the adversary access from black-box to white-box (\autoref{sec:anatomy_access}).

\subsection{Selecting Challenges and Datasets}
\label{sec:anatomy_choosing}

An important aspect of any privacy game is how the challenges and datasets are selected.
In \autoref{game:MI_example}, the challenge point is a single record $z$; in other games, the challenge could comprise multiple points or even a data distribution. For the discussion below, we simplify the language by talking about a single challenge point.
We discuss below three methods commonly used in the literature.

\boldparagraph{Randomly sampled}
The challenge is sampled from a distribution by the challenger as part of the game~\cite{Yeom:2020,Wang:2021,Humphries:2021}.
A randomly sampled challenge provides a measure of \emph{average case} privacy.
While average case privacy measures the risk for average users, the risk for outliers can be significantly higher.

\boldparagraph{Externally provided}
The challenge is provided as a parameter of the game~\cite{Mahloujifar:2021,Humphries:2021}.
This may be used to measure privacy of specific points, \ie, it provides \emph{individual case} privacy.

\boldparagraph{Adversarially chosen}
The challenge is selected by the adversary during the game~\cite{Chang:2021,Malek:2021,Carlini:2021b}.
Since the adversary can select the most advantageous challenge based on the information provided, this provides a measure of \emph{worst case} privacy, \ie, measuring the risks for all users including outliers.
For example, a strong membership inference adversary could choose a challenge that is an outlier \wrt the training data distribution, so that a target classification model is unlikely to classify it correctly unless it is included in the training dataset.
This setting is usually considered when auditing a system to identify risks.

\boldparagraph{Additional considerations}
When the challenge is externally provided or adversarially chosen, the parameters of the game cannot completely determine a correct adversary guess.
Otherwise, security statements that universally quantify over adversaries are void because the quantification includes adversaries with a hardcoded correct guess.
This is similar to the difficulty of defining collision resistance of hash functions~\cite{Rogaway:2006}.

\boldparagraph{Selecting datasets}
The training dataset can also be selected using any of the three options above: it can be randomly sampled by the challenger, externally provided, or (partially) chosen by the adversary.
The latter can be used to represent the case where the model has been trained on (poisoned) data contributed by potentially malicious users~\cite{Mahloujifar:2022,Tramer:2022}.

\subsection{Adversary Access}
\label{sec:anatomy_access}

Depending on the scenario, the adversary may have different levels of access to the target model, training algorithm, training distribution, and training dataset.
This allows the game to capture different threat models, which should ideally match the known or assumed capabilities of real-world adversaries.
Most games assume one of two settings: \emph{black-box} or \emph{white-box} access.

\boldparagraph{Black-box}
In this scenario, the adversary only has query access to the target model (\eg, a cloud-hosted model with an inference API)~\cite{Carlini:2022}.
To formalize this setting, we give the adversary access to the model through an oracle
$\mathbf{Oracle}~\Oracle^\theta(x): \Return~\theta(x)$.
This allows the adversary to query the model $\theta$ on inputs of their choosing and observe the responses, but does not reveal internal workings of the model, such as its architecture or weights.
Depending on the scenario, the oracle can return a confidence for each label, or only the highest-confidence label~\cite{Li:2021,Choquette-Choo:2021}. The latter setting matches inference APIs that do not reveal confidence values, like some email spam classifiers or auto-completion systems.
Additionally, the oracle can be instrumented to post-process responses, or to only emit responses for queries satisfying a (stateful) predicate, \eg, to enforce a bound $N$ on the number of allowed queries the challenge can initialize $q_0 = 0$ and provide
\begin{align*}
\begin{array}{l}
\textbf{Oracle } \Oracle_N^\theta(x) \\
\qquad	q_\theta \gets q_\theta + 1\\
\qquad \textbf{if } q_\theta \leq N \textbf{ then } \Return \arg\max \theta(x) \textbf{ else } \Return~\bot
\end{array}
\end{align*}

\boldparagraph{White-box}
The white-box setting represents the strongest adversary, who has full direct access to the target model \ie, $\A(\theta,\ldots)$.
This obviously provides the adversary with all the capabilities of the black-box setting, but also allows the adversary to inspect the internals of the model including its trained weights~\cite{Sablayrolles:2019,Leino:2020}.
For instance, a model deployed on clients' devices gives white-box access to malicious clients. Alternatively, a successful black-box model stealing attack would enable an adversary to operate in a white-box setting.

\boldparagraph{Grey-box}
In between the black-box and white-box settings, there is a range of \emph{grey-box} threat models in which the adversary has more than black-box but less than full white-box access to the target model.
For example, the adversary could know the architecture of a target model, some of its training hyperparameters, or the public model from which the model has been fine-tuned~\cite{Shokri:2017,Salem:2019}.
Such extra information can be the output of a hyperparameter stealing attack~\cite{Wang:2018}.

\boldparagraph{Auxiliary information}
In addition to having access to the target model, an adversary may have auxiliary information that could be useful for certain attacks.
For example, most MI attacks assume the adversary has access to auxiliary data distributed similarly to the target model's training data, \eg, for building shadow models. This is captured in games by giving the adversary the distribution from which the training data was sampled.

\boldparagraph{Resource constraints}
Most game-based formulations do not explicitly limit the resources available to an adversary, \ie, they consider information-theoretic adversaries.
It could be important to consider resource-limited adversaries that can only issue a specific number of queries to an oracle, or can use a certain amount of memory, or are otherwise computationally bounded.
Intuitively, limiting these resources can reduce the effectiveness of an attack.
These limitations can be specified outside the game as constraints on the adversary, enforced by instrumenting the code of the game (as in Oracle $\Oracle_N^\theta$ above), or incorporated into the measure of success.

\subsection{Measuring Adversary Success}
\label{sec:anatomy_evaluating}

There are various ways of quantifying the adversary's success in games.
We discuss commonly used metrics next.

\subsubsection*{Attack Success Rate}

The \emph{attack success rate} (ASR) measures the expected number of times the adversary succeeds (\ie, wins the game)
over multiple runs.
ASR is arguably the most intuitive and widespread metric for quantifying adversary success;
for example, it matches the attacker's \textit{accuracy} in membership inference.

However, the main drawback of ASR is that it does not take into account the baseline success probability for a given task.
For example, if we evaluate an ML model's resilience to attribute inference, the prior distribution of that attribute will play a role in the adversary's success.
For instance, if the attribute can only take one value, it is trivial for an adversary to achieve 100\% ASR, but this will not be a meaningful measure.
Similarly, the prior probability that an example belongs to the training set affects membership inference accuracy.

Ideally, the metric should quantify the success of an adversary relative to a suitable \emph{baseline}.
The baseline should represent the \emph{a priori} adversary success rate; that is, it should quantify the adversary's success rate if they used only their prior knowledge and had no access to the model.

\subsubsection*{Adversary Advantage}

The notion of \emph{advantage} is a commonly used metric in cryptography, which relates an adversary's success rate to a baseline.
This gives a better intuition of how much an adversary gains by having access to the model (in any of the forms defined in \autoref{sec:anatomy_access}).
In general terms, suppose the adversary is trying to infer some variable $p$; this could be the membership of a data record or the value of a coin toss.
If $\Pr[\A = p]$ is the adversary's success rate (probability to guess $p$ correctly), and $G$ is the baseline success rate, the advantage can be expressed as
$\Adv(\A) = \nicefrac{\Pr[\A = p]-G}{1-G}.$
Assuming $\Pr[\A = p] \geq G$, this metric quantifies the adversary's advantage on a scale of $[0,1]$ relative to the baseline $G$; $0$ represents no advantage over the baseline and $1$ is a perfect attack.
When the secret information $p$ is binary with a uniform prior, $G=\nicefrac{1}{2}$.
This leads to the familiar expression $\Adv(\A) = 2 \Pr[\A = p] - 1.$
Advantage is commonly used as a metric for ML privacy attacks.
For example, \citet{Yeom:2020} define the MI advantage for an adversary $\A$ as follows:
\begin{equation*}
  \Adv_\MI(\A,\Train,n,\D) = 2 \Prob{\MI(\A,\Train,n,\D)}{\guess{b} = b} - 1,
\end{equation*}
where \MI is the membership inference experiment in \autoref{game:MI_example}, and $\Prob{G}{E}$ denotes the probability of event $E$ in the probability space defined by game $G$.

Providing an adequate baseline may be difficult because it may not be possible to accurately model the adversary's knowledge.
This issue can often be bypassed by careful design of the game.
For example, instead of asking the adversary to reconstruct an arbitrary attribute's value, the game can be designed such that the adversary must distinguish between two equally-likely values of the attribute.

\subsubsection*{Beyond advantage}

Average case metrics such as ASR fail to capture inference risks for individuals or subpopulations.
For example, a MI attack against a model may achieve roughly 50\% accuracy (with a 50\% baseline) on average across the population, yet the same attack may perform better when targeting specific individuals or subpopulations~\cite{Chang:2021,Kulynych:2022}.
Having raised similar concerns, \citet{Carlini:2022} suggest that an adversary should be considered successful if it reliably succeeds even on small number of cases. For instance, a MI attack that achieves a high true positive rate (TPR) at some low false positive rate (FPR) could be consequential even if it has low accuracy.

In this paper, we focus on advantage as a metric, since it has the following benefits:
\begin{inparaenum}[(1)]
\item it has an easy interpretation---it represents the gain of an adversary from having access to the system under scrutiny
versus an adversary with only prior knowledge;
\item it is directly related to other metrics, such as ASR (which can be derived directly from it), true and false positive rates (\eg, \cite{Yeom:2020}), and Differential Privacy~\cite{Humphries:2021,Chatzikokolakis:2020};
\item if the attacker's challenge is binary (\eg, distinguishing between members and nonmembers), the advantage computed
when assuming the two choices have a uniform prior gives a bound for any other prior~\cite{Chatzikokolakis:2020}.
Nevertheless, given a game formulation, one can consider other metrics of interest: \eg, area under the ROC curve (AUC-ROC), F1-score, and TPR  at fixed FPR thresholds~\cite{Carlini:2022}.
\end{inparaenum}

\subsection{Consequences of Attacks}

The anatomy we presented can be used to specify threat models and quantify the chances that an adversary successfully achieves their goal. However, the \emph{consequences} of a successful attack depend less on the threat model but rather on the adversary's goal (\autoref{sec:anatomy_goals}) and on the design of the ML system, \eg, the sensitivity of the training data.
For example, the consequences of successful membership inference will be the same irrespective of whether it was performed in a black-box or white-box setting.

\section{Game-based Formalization of Inference Risks}
\label{sec:formalization}
In this section we present privacy games for the five adversary goals introduced in \autoref{sec:anatomy_goals}.
We summarize the notation in \autoref{tab:notation} and the threat models considered in all games in \autoref{tab:AllGames}.

\begin{table}[htb]
\caption{Summary of notation}
\label{tab:notation}
\centering
\begin{tabular}{@{}p{45pt}@{~~}l@{}}
\toprule
\textbf{Notation}              & \textbf{Description} \\
\midrule
$\Train$                       & A stochastic training algorithm \\
$\D$                           & A distribution over examples \\
$\D^n$                         & Distribution of $n$ independent examples from \D \\
$\A$, $\A'$                    & Adversary procedures sharing mutable state \\
$z \sim \D$                    & Draw an example $z$ from \D \\
$S \sim \D^n$                  & Draw $n$ examples $S$ independently from \D \\
$b \sim \bit$                  & Sample a bit $b$ uniformly \\
$b \sim 0 \oplus_p 1$          & Sample 0 with probability $p$, 1 with probability $1-p$\\
$y \gets \mathcal{P}(\vec{x})$ & Call $\mathcal{P}$ with arguments $\vec{x}$ and assign result to $y$ \\
\bottomrule
\end{tabular}
\end{table}

\subsection{Membership Inference}
\label{sec:membership_inference}

Membership inference aims to predict the participation of an entity in the training dataset of the model.
The first (record-level) membership inference attack on supervised learning was proposed by \citet{Shokri:2017} against ML-based classifiers.
Subsequent work has explored membership inference attacks with differing degrees of access to the model (\eg, white-box~\cite{Sablayrolles:2019,Leino:2020} or label-only attacks~\cite{Li:2021,Choquette-Choo:2021}), against different types of models (\eg, generative models~\cite{Chen:2020, Hilprecht:2019, Hayes:2019}, image segmentation~\cite{He:2020}, contrastive learning~\cite{Liu:2021b}, recommender systems~\cite{Zhang:2021}, and Graph Neural Networks (GNN)~\cite{Wu:2021}), and under entirely different threat models~\cite{Shejwalkar:2021,Salem:2019,Hui:2021}.

We present MI variants that have been formalized as games.
We divide the games into two categories depending on whether they focus on a single record (\emph{record-level}) or a \emph{user} represented by a collection of records (\emph{user-level}).

\subsubsection*{Record-level Membership Inference}

The most common interpretation of record-level membership inference is given by the game introduced by \citet{Yeom:2020}, which we presented as \autoref{game:MI_example} in \autoref{sec:anatomy}.
\autoref{game:MI} below presents a semantically equivalent reformulation \MI.
The reader can verify that $b, \theta, z_0$ are distributed identically to $b, \theta, z$ in \autoref{game:MI_example} and thus the joint distribution of $b, \guess{b}$ is the same in both games.
This game considers an adversary with white-box access to the model---they have the model at their disposal and can query it freely, analyze its architecture and parameters, and observe its dynamic behavior.
Since the training dataset and the challenge $z_0$ are sampled from $\D$, this game measures average case MI resilience.

\begin{game}
\caption[F]{\boxA{$\MI^{\phantom{|}}$} \boxB{$\MI^{\textsf{skew}}$} {\boxC{$\MI^\Adv$}}}
\label{game:MI}
\LinesNumbered
	\KwIn{$\Train, \D, n, \boxB{$p$}, \boxC{$\A'$}, \A$}
	$S \sim \D^{n-1}$\;
	$b \sim \boxA{$\bit$} \,\, \boxB{$0 \oplus_p 1$} \,\, \boxC{$\bit$}$\;
	$z_0 \sim \boxA{\D}\,\, \boxB{\D}\,\, \boxC{$\A'(\Train, \D, n)$}$\;
	$z_1 \sim \D$\;
	$\theta \gets \Train(S \cup \{z_b\})$\;
	$\guess{b} \gets \A(\Train, \D, n, \boxB{$p$}, \theta, z_0)$
\end{game}

Several variants of the basic MI game have been considered in the literature; some are semantically equivalent (\eg, \cite{Kulynych:2022,Humphries:2021}) whilst others alter its semantics.
We next systematize these latter variants using the anatomy presented in \autoref{sec:anatomy}.

\citet{Jayaraman:2021} consider game $\MI^\textsf{skew}$ which generalizes \MI by introducing a parameter $p$ representing the prior membership probability (\autoref{game:MI}, line 2). The original \MI game assumes a balanced prior and is recovered as a special case when $p = \nicefrac{1}{2}$.

\citet{Chang:2021} consider game $\MI^\Adv$ in \autoref{game:MI} which strengthens the adversary by allowing them to select the challenge point (line 3).
This game measures worst case MI resilience for an average dataset, \ie, resilience against this variant protects all records---even outliers---against MI.
See \SMI in \autoref{game:DPD} for an even stronger attack where $S$ is adversarially chosen.

\citet{Carlini:2022} consider game $\MI^\textsf{BB}$ which differs in two aspects from \MI.
Firstly, it assumes a black-box adversary who is given only inference access to the model through an oracle, $\textbf{Oracle}~\Oracle^\theta(x): \Return~\theta(x)$ (modifying line 9 in \autoref{game:MI_example}).
This is appropriate when the target model is hosted in the cloud or in a trusted execution environment that ensures its confidentiality.
Secondly, rather than sampling the challenge point from $\D$ when $b = 1$, the challenger samples it from $\D \setminus S$ (modifying line 6 in \autoref{game:MI_example}), thus excluding the case where the challenge happens to be in $S$ by chance.
This is in contrast to game \MI, where nonmembers are sampled from the complete distribution and \textit{may} be contained in $S$.
While doing this seems intuitive, \citet[p.41]{Yeom:2020} note that it is problematic since an adversary could gain advantage not through access to the model but rather by analyzing $\D$ to infer which points are more likely to have been sampled into $S$.
For instance, consider a distribution $\D$ with support $\{x_0,\dotsc,x_m\}$ that assigns probability \nicefrac{1}{2} to $x_0$ and $\nicefrac{1}{2m}$ to each of $x_1,\dotsc,x_m$.
An adversary that ignores $\theta$ and guesses $\guess{b} = 0$ if and only if $z = x_0$ has advantage greater than $\nicefrac{1}{2} - \nicefrac{1}{2^n}$.

\citet{Tramer:2022} introduce a generic privacy game where the goal of the adversary is to guess which point from a universe $\mathcal{U}$ has been included in the training dataset of the target model.
They present variants with ($\MI^\textsf{Pois}$) and without ($\MI^\textsf{Diff}$) poisoning, shown in \autoref{game:MIabs}.
$\MI^{\textsf{Pois}}$ lets the adversary statically poison part of the training dataset (\autoref{sec:anatomy_choosing}).
By considering $\mathcal{U} = \{\hat{z}, \bot\}$, where $\bot$ indicates the absence of an example, the generic game can represent a black-box membership inference attack for a fixed externally provided target example $\hat{z}$.
Compared to variants of membership inference discussed previously, this results in training datasets of different sizes depending on the outcome of sampling the challenge $z$: \eg in $\MI^\textsf{Diff}$ the model may be trained on $S \cup \{\hat{z}\}$ or just on $S$.
This usually does not make a significant difference as training datasets are large and models do not leak the size of their training dataset.
As in $\MI^\textsf{BB}$, values in $S$ are excluded when sampling $z$, which leads to similar problems.

\begin{game}
\caption[F]{$\MI^{\textsf{Diff}}$ \boxA{$\MI^{\textsf{Pois}}$}}
\label{game:MIabs}
	\KwIn{$\Train, \D,\mathcal{U}, n, \A$, \boxA{$\A^\prime, n^\prime$}}
	$S \sim \D^n$\;
	$z \sim \mathcal{U} \setminus S$\;
	\boxA{$S^\prime \gets \A^\prime(\Train, \D, \mathcal{U}, n^\prime)$} \tcp*[r]{$|S^\prime| = n^\prime$}

	$\theta \gets \Train(S \cup \{z\}$ \boxA{$\cup\, S^\prime$}$)$\;
	$\guess{z} \gets \A(\Train, \D, \mathcal{U}, n, \Oracle^\theta(\cdot),$\boxA{$S^\prime$}$)$\;

	\BlankLine

	\textbf{Oracle} $\Oracle^\theta(x)$:
	\Return $\theta(x)$
\end{game}

\boldparagraph{Other variants}
\citet{Humphries:2021} sample the training dataset and challenge point from different distributions (\autoref{game:MM}); we use this variant as the basis for our case study in \autoref{sec:case_study}.
\citet{Tang:2022} define single-query membership inference games where the adversary is only given the output of the trained model on the challenge point, but where the adversary selects the set of examples from where the training dataset is subsampled (see \autoref{sec:SQMI} in the Appendix).
\citet{Gao:2022} consider \emph{deletion inference}, a variant of membership inference in the setting of \emph{machine unlearning}, where
the adversary is given access to a model before and after one of two examples is deleted and is asked to guess which example was deleted.

\subsubsection*{User-level Membership Inference}

Privacy laws such as GDPR require generalizing the goal of MI. Instead of focusing on a single record, the interest is now the complete data of an individual. For instance, an auditor would be interested in learning if a user's data---usually modeled as a collection of records---was used to train a target model. User-level membership inference was introduced to model such scenarios.
\citet{Mahloujifar:2021} formalize user-level MI as in~\autoref{game:MIuser}.
They consider a meta-distribution \D from where $m$ user distributions are sampled. The adversary targets a particular user contributing a dataset $S^*$. This game presents the adversary with a task easier than \autoref{game:MI_example} since they must infer whether an entire group of records is within the training dataset, \ie, it measures \emph{group} privacy.

\begin{game}
\caption[F]{$\MI^\mathsf{User}$}
\label{game:MIuser}
	\KwIn{$\Train, \D, n, \A, S^*, m$}
	$b \sim \bit$\;
	$\D_1, \dotsc, \D_{m} \sim \D$\;
	\For{$i = 1, \dotsc, m-1$}{
		$S_i \gets \D^n_i$\;
	}
	\eIf{$b = 0$}{
		$S_m = S^*$
	} {
		$S_m\gets \D^n_m$
	}
	$\theta \gets \Train(\bigcup_{i=1}^m S_i)$\;
	$\guess{b} \gets \A\left(\Train, \D, n, \Oracle^\theta(\cdot), S^*, m\right)$\;

	\BlankLine

	\textbf{Oracle} $\Oracle^\theta(x)$:
	\Return $\theta(x)$
\end{game}

\subsection{Attribute Inference}
\label{sec:AI}

In attribute inference (AI) attacks, the adversary aims to infer a sensitive attribute of a target record.
\citet{Wu:2016} were the first to formalize AI, confusingly under the name of \emph{model inversion}.
We follow here the more general formalization given by \citet{Yeom:2020} shown in \autoref{game:AI}. 
Recently the scope of AI expanded to other settings~\cite{Zhang:2022,Jayaraman:2022}.

\begin{game}
\caption[F]{\boxA{$\AI$} \boxB{$\Inv$}}
\label{game:AI}
\KwIn{$\Train, \D, n, \A, \varphi, \pi$}
	$S \sim \D^n$\;
	$b\sim \bit $\;
	\eIf{$b = 0$}{
		$z \sim$ \boxA{$S$} \boxB{$\D$}
	}
	{
		$z \sim \D$
	}
	$\theta \gets \Train(S)$\;
	$\guess{a} \gets \A(\Train, \D, n, \theta, \varphi(z))$
\end{game}

In the $\AI$ game, $\varphi(z)$ denotes the adversary's knowledge about the challenge $z$, and $\pi$ a function that extracts the information targeted by the attack, \eg, if $t$ represents the target sensitive attributes, then $\pi(z) = t$.
The experiment is similar to the basic membership inference experiment (\autoref{game:MI_example}) except for the information that the adversary is given and the winning condition.
The adversary is given $\varphi(z)$ and aims to infer $\pi(z)$.
The adversary wins if it correctly predicts these attributes, \ie, $\guess{a} = \pi(z)$.
Training data poisoning can be considered by including adversarially chosen data when training the target model as done for \MI in \autoref{game:MIabs} ($\MI^{\textsf{Pois}}$), an instance of the generic game of \citet{Tramer:2022}.

\boldparagraph{Model inversion}
Another adversary goal with a similar aim to \AI is model inversion~\cite{Wang:2021}.
Model inversion attacks were introduced by \citet{Fredrikson:2015} and subsequently formalized by \citet{Wang:2021} (\Inv in \autoref{game:AI}).
The difference between attribute inference and model inversion according to \citet{Wang:2021} is in how the challenge is sampled: in \AI it is sampled from the training dataset, while in \Inv it is sampled from the distribution \D.
While AI measures privacy risk for members of a model's training dataset, model inversion measures the privacy loss of publishing the model for members of the underlying population.
Whether this is considered a privacy risk is up to debate: a successful attack may lead to the adversary learning information from records that are not part of the training dataset or that do not even exist.
Model owners concerned only with the privacy of the training dataset would use the \AI game, whilst those concerned about population privacy would prefer \Inv.

\subsection{Reconstruction}
\label{sec:reconstruction}

Reconstruction attacks aim to recover entire examples in the training dataset of a model.
Reconstruction has been studied in various settings, including Graph Neural Networks~\cite{Zhang:2022}, image classification~\cite{Salem:2020}, and text generation~\cite{Carlini:2019,Zanella-Beguelin:2020,Carlini:2021c}.
A distilled scenario, where the adversary learns the training data of the target model except for a target example was first formalized by \citet{Balle:2022} as experiment \RC in \autoref{game:RC}.

\begin{game}
\caption[F]{\boxA{\RC} \boxB{$\RC^\textsf{Ran}$}}
\label{game:RC}
\KwIn{\boxA{$S$}, \boxB{$\D, n$}, $\pi, \Train, \A$}
	\boxB{$S \sim \D^{n-1}$}\;
	$z \sim \pi$\;
    $\theta \gets \Train(S \cup \{z\})$\;
    $\guess{z} \gets \A(\Train, \theta, \boxB{$\D, n$}, S)$
\end{game}

Reconstruction robustness is parametrized by bounds on the error and success probability and defined as follows.

\begin{definition}[\citet{Balle:2022}, Definition 2]
A training pipeline is $(\eta, \gamma)$-reconstruction robust with respect to a prior $\pi$ and reconstruction loss $\ell$ if for any dataset $S$ and any reconstruction adversary $\A$,
\begin{equation*}
\Prob{\RC}{\ell(z, \guess{z}) \leq \eta} \leq \gamma
\end{equation*}
\end{definition}

The adversary is given the model $\theta$, training algorithm $\Train$, and the training dataset $S$ except for one point $z$ which they need to reconstruct.
Game $\RC^\textsf{Ran}$ models how other points in the training dataset are sampled, instead of considering a fixed dataset $S$.
The advantage of an adversary \A against \RC \wrt a baseline that ignores $\theta$ and just samples $\guess{z}$ from \D is
\begin{equation*}
  \Adv_{\RC}(\A) =
    \Prob{\RC}{\guess{z} = z} -
    \Prob{z,\guess{z} \sim \D}{\guess{z} = z}
\end{equation*}

Alternatively, one can consider the baseline success of an adversary that picks $\guess{z}$ according to $\pi$,
\begin{equation}
\label{eq:RC_baseline}
    \sup_{\guess{z} \in \supp(\pi)} \Prob{z \sim \pi}{\ell(z,\guess{z}) \leq \eta}
\end{equation}
Both games can be adapted to consider a poisoning-capable adversary as demonstrated in \autoref{game:MIabs}.

\boldparagraph{Reconstruction in language models}
Recent work focused on large language models and evaluated reconstruction attacks against them.
Attacks can be categorized as untargeted~\cite{Carlini:2021c} or targeted~\cite{Carlini:2019}.
Untargeted attacks aim to reconstruct \textit{any} training data from the generative model, whilst targeted attacks aim to reconstruct \textit{specific} training data records, which may have been inserted as canaries during training.
To demonstrate the flexibility of privacy games, we formalize an example from each category, as shown in \autoref{game:RC_targ}.

We formalize a black-box untargeted data reconstruction attack by \citet{Carlini:2021c} tailored to large generative language models as $\RC^\textsf{Untarg}$.
The authors measure the success of an attack by its true positive rate or recall, that is, the fraction of examples in $\Guess{S}$ that are in the training dataset $S$.

We formalize a black-box targeted reconstruction attack by \citet{Carlini:2019} as $\RC^\textsf{Targ}$.
The authors insert a \emph{canary} multiple times into the training data as a way to measure unintended memorization in generative models.
Canaries are specified by a format sequence $s[\cdot]$ that fixes some tokens and leaves \emph{holes} to be filled with secrets sampled from a randomness space $\mathcal{R}$.
For example, $s = $ ``the PIN is \ds\,\ds\,\ds\,\ds'' with $\mathcal{R}$ being the space of 4-digit decimal numbers.
\citet{Carlini:2019} measure the success of targeted canary reconstruction as the reduction in the guessing entropy of secrets in canaries given the model.

\begin{game}
\caption[F]{$\RC^\textsf{Untarg}$ \boxA{$\RC^\textsf{Targ}$}}
\label{game:RC_targ}
\KwIn{$\Train, \D, n, \A$, \boxA{$\mathcal{R}, s, m$}}
$S \sim \D^n$\;
\boxA{$r \sim \mathcal{R}$}\;
$\theta \gets \Train(S $ \boxA{$\cup \{ s[r] \}^m$} $)$\;
$\Guess{S} \gets \A(\Train, \D, n, \Oracle^\theta(\cdot), $ \boxA{$\mathcal{R}, s$}$)$\;

\BlankLine

\textbf{Oracle} $\Oracle^\theta$(x):
\Return $\theta(x)$
\end{game}

\boldparagraph{Selecting a game}
Game \RC is appropriate when evaluating the worst case risk of reconstructing an example in the training dataset. It conservatively considers an informed adversary that knows all examples but the target, and incorporates the adversary's background knowledge in a prior.
Game $\RC^\textsf{Ran}$ considers an equally informed adversary, but averages the reconstruction risk over the choice of other training examples.
Game $\RC^\textsf{Untarg}$ represents a more realistic threat model and should be chosen when evaluating the risk of indiscriminately reconstructing training data, while $\RC^\textsf{Targ}$ is appropriate for auditing the risk of extracting data following certain patterns.

\boldparagraph{Other variants}
Similar to MI, reconstruction attacks have been adapted to the machine unlearning setting.
\citet{Gao:2022} consider \emph{deletion reconstruction}, where an adversary is given access to a model before and after a random training example is deleted and is asked to reconstruct it.

\subsection{Distribution Inference}

Distribution inference attacks do not focus on specific data records, but instead aim at inferring properties about the training data distribution.
We next describe two variants of distribution inference.
The first is property inference, \eg, where the adversary is interested in learning about the prevalence of specific sensitive attributes in the training data, such as sex or ethnicity.
The second is subject-level distribution inference, where the training data is sampled from a mixture of distributions, each corresponding to a subject that may participate in training. The adversary's goal is to infer whether a subject has participated knowing the subject's data distribution rather than concrete samples like in game $\MI^\textsf{User}$ in \autoref{game:MIuser}.

\subsubsection*{Property Inference}

Property inference attacks were first proposed by \citet{Ganju:2018} in the white-box setting and by \citet{Zhang:2021b} in the black-box setting.
\citet{Zhou:2022} showed them to be effective against generative models and GANs specifically.
\citet{Suri:2022} formalized property inference attacks as \PI in \autoref{game:PI}, parametrized by two functions $\gzero,\gone$ that transform an underlying distribution.

\begin{game}
\caption[F]{\boxA{$\PI^{\phantom{|}}$} \boxB{$\PI^\textsf{Gen}$}}
\label{game:PI}
\KwIn{$\boxA{$\D, \gzero, \gone$}\,\, \boxB{$\D_0, \D_1$}, n, \Train, \A$}
	$b\sim \bit $\;
	$S \sim $\boxA{$\mathcal{G}_b(\D)^n$} \boxB{$\D_b^n$}\;
	$\theta \gets \Train(S)$\;
	$\guess{b} \gets \A(\Train, \boxA{$\D, \gzero, \gone$}\,\, \boxB{$\D_0, \D_1$}, n, \theta)$
\end{game}

$\PI^\textsf{Gen}$ is an equivalent formulation parametrized by two distributions corresponding to the application of $\gzero$, $\gone$ to the base distribution \D in \PI.
\citet{Hartmann:2023} generalize this to more than two distributions.

Similarly to \MI and \AI, poisoning can be modelled as in \autoref{game:MIabs} by letting the adversary choose part of the training dataset of the target model.
\citet{Mahloujifar:2022} and \citet{Chaudhari:2022} show that poisoning increases inference risk by injecting data to maximize leakage of properties of the training dataset. For instance, in multi-party learning, a malicious participant may contribute poisoned data crafted to amplify property leakage of data from other participants.

\subsubsection*{Subject-level Distribution Inference}

Subject-level distribution inference broadens the scope of user-level membership inference by not assuming access to the user's exact data that may have been used to train a model.
Instead, it only requires the adversary know the distribution from which the target user's data is sampled.
\citet{Suri:2022b} present subject membership inference as a special case of distribution inference.
We similarly formalize subject-level inference in \autoref{game:MIsubj}.

\begin{game}
\LinesNumbered
\caption[F]{$\MI^\textsf{Subj}$}
\label{game:MIsubj}
\KwIn{$\Train, \D, \D_*, n, m, \A$}
	$b \sim \bit$\;
	$\D_1, \dotsc, \D_m \sim \D$\;
	\For{$i = 1, \dotsc, m-1$}{
		$S_i \sim {\D}_i^n$
	}
	\eIf{$b = 0$}
	{
		$S_m \sim \D_*^n$
	}
	{
		$S_m \sim \D_m^n$
	}
	$\theta \gets \Train(\bigcup_{i=1}^m S_i)$\;
	$\guess{b} \gets \A\left(\Train, \D, \D_*, n, m, \theta\right)$
\end{game}

The training data distribution is structured as a mixture of distributions corresponding to a set of subjects.
This is a property inference attack because the adversary seeks to infer which of two distributions the training data is sampled from.
However, conceptually, the adversary's goal is to infer membership of a subject's data since the only difference between the two distributions is the presence of the target subject in the mixture.

A successful subject-level distribution inference attack can identify if a user's data was used to train the target model without knowing which exact examples were used; \ie, with access to only the user's data distribution and not the sampled dataset as in \autoref{game:MIuser}.

\subsection{Differential Privacy Distinguishability}

Differential Privacy Distinguishability (DPD) formalizes the threat model underlying the definition of DP, where the adversary aims to distinguish between models trained on adjacent datasets.
We formalize as game \DPD in \autoref{game:DPD} the variant corresponding to the \emph{substitute one} adjacency relation, where two datasets are adjacent if one can be obtained from the other by substituting a single record.
The \DPD game represents a worst-case variant of the membership inference game \MI where the training data and challenges are adversarially chosen.

Prior work used DP distinguishing attacks to statistically estimate or audit the privacy of training pipelines~\cite{Carlini:2021b,Jagielski:2020,Zanella-Beguelin:2022,Tramer:2022b}.
\citet{Marathe:2021} define subject-level differential privacy by considering datasets as adjacent when they differ in the data of a user, which can be seen as a counterpart to user-level membership inference.
\citet{Humphries:2021} and \citet{Balle:2022} discuss strong membership inference, a threat model in between DPD and MI.
In this game, formalized as \SMI in \autoref{game:DPD}, the adversary knows but does not choose the two adjacent datasets.
As mentioned in \autoref{sec:anatomy_choosing} this narrows the scope of the measured privacy, \eg, from worst to individual case privacy.

\begin{game}
\caption[F]{\boxA{\DPD} \boxB{\SMI}}
\label{game:DPD}
\label{game:SMI}
\KwIn{$\Train, \A, \boxA{$\A', n$}, $\boxB{$S, z_0, z_1$}}
    \boxB{$S, z_0, z_1 \gets \A'(\Train, n)$}  \tcp*[r]{$|S| = n-1$}
    $b\sim \bit$\;
    $\theta \gets \Train(S \cup \{z_b\})$\;
    $\guess{b} \gets \A(\Train, \theta, S, z_0, z_1)$
\end{game}

\begin{table*}[!t]
\centering
\caption{An overview of different games and features of their corresponding threat models.\\
\rm
  \checkmark\ indicates the game has this feature,
  \nomark\ indicates the game does not have this feature,
  \xmark\ indicates that the feature is not applicable.}
\label{tab:AllGames}
\rowcolors{2}{white}{gray!25}
\begin{tabular}{l@{~~~}lccccccccccc}
\toprule

\rowcolor{white}
& & \multicolumn{2}{c}{Adversary Access}& \multicolumn{3}{c}{Challenge}& \multicolumn{3}{c}{Training Dataset}& \multicolumn{3}{c}{Adversary Interest}\\

\cmidrule(lr){3-4} \cmidrule(lr){5-7} \cmidrule(lr){8-10} \cmidrule(lr){11-13}

\rowcolor{white}
Game & Definition & Black-box & White-box & Rand & Adv & Param & Rand & Adv & Param & Record & Subject & Distribution \\

\midrule

\rowcolor{white}
\multicolumn{13}{c}{\textbf{Membership Inference}} \\

$\MI$                 & \autoref{game:MI}      \cite{Yeom:2020,Kulynych:2022,Humphries:2021} & \nomark & \checkmark & \checkmark & \nomark & \nomark & \checkmark & \nomark & \nomark & \checkmark & \nomark & \nomark \\
$\MI^\textsf{skew}$   & \autoref{game:MI}      \cite{Jayaraman:2021} & \nomark & \checkmark & \checkmark & \nomark & \nomark & \checkmark & \nomark & \nomark & \checkmark & \nomark & \nomark \\
$\MI^\textsf{BB}$     & \autoref{game:MI}      \cite{Carlini:2022} & \checkmark & \nomark & \checkmark & \nomark & \nomark & \checkmark & \nomark & \nomark & \checkmark & \nomark & \nomark \\
$\MI^{\Adv}$          & \autoref{game:MI}      \cite{Chang:2021} & \nomark & \checkmark & \nomark & \checkmark & \nomark & \checkmark & \nomark & \nomark & \checkmark & \nomark & \nomark \\
$\MI^\textsf{Diff}$   & \autoref{game:MIabs}   \cite{Tramer:2022} & \checkmark & \nomark & \checkmark & \nomark & \nomark & \checkmark & \nomark & \nomark & \checkmark & \nomark & \nomark \\
$\MI^\textsf{Pois}$   & \autoref{game:MIabs}   \cite{Tramer:2022} & \checkmark & \nomark & \checkmark & \nomark & \nomark & \checkmark & \checkmark & \nomark & \checkmark & \nomark & \nomark \\
$\MI^\textsf{User}$   & \autoref{game:MIuser}  \cite{Mahloujifar:2021} & \checkmark & \nomark & \nomark & \nomark & \checkmark & \checkmark& \nomark & \nomark & \nomark & \checkmark & \nomark \\
$\MM$                 & \autoref{game:MM}      \cite{Humphries:2021} & \nomark & \checkmark & \checkmark & \nomark & \nomark & \checkmark & \nomark & \nomark & \checkmark & \nomark & \checkmark \\
$\MI^\textsf{SQ}$     & \autoref{game:SQMI}    \cite{Tang:2022} & \checkmark & \nomark & \checkmark & \nomark & \nomark & \nomark & \checkmark & \nomark & \checkmark & \nomark & \nomark \\

\rowcolor{white}
\addlinespace
\multicolumn{13}{c}{\textbf{Attribute Inference and Model Inversion}} \\

$\AI$                 & \autoref{game:AI}      \cite{Yeom:2020} & \nomark & \checkmark & \checkmark & \nomark & \nomark & \checkmark & \nomark & \nomark & \checkmark & \nomark & \nomark \\
$\Inv$                & \autoref{game:AI}      \cite{Wang:2021} & \nomark & \checkmark & \checkmark & \nomark & \nomark & \checkmark & \nomark & \nomark & \checkmark & \nomark & \nomark \\

\rowcolor{white}
\addlinespace
\multicolumn{13}{c}{\textbf{Data Reconstruction}} \\

\rowcolor{gray!25}
$\RC$                 & \autoref{game:RC}      \cite{Balle:2022} & \nomark & \checkmark & \checkmark& \nomark & \nomark & \nomark & \nomark & \checkmark & \checkmark & \nomark & \nomark \\
\rowcolor{white}
$\RC^\textsf{Untarg}$ & \autoref{game:RC_targ} \cite{Carlini:2021c} & \checkmark & \nomark & \xmark & \xmark & \xmark & \checkmark& \nomark & \nomark & \checkmark & \nomark & \nomark \\
\rowcolor{gray!25}
$\RC^\textsf{Targ}$   & \autoref{game:RC_targ} \cite{Carlini:2019} & \checkmark & \nomark & \checkmark& \nomark & \nomark & \checkmark & \nomark & \nomark & \checkmark & \nomark & \nomark \\

\rowcolor{white}
\addlinespace
\multicolumn{13}{c}{\textbf{Distribution Inference}} \\

\rowcolor{gray!25}
$\PI$                & \autoref{game:PI}      \cite{Suri:2022} & \nomark & \checkmark & \xmark & \xmark & \xmark & \checkmark & \nomark & \nomark & \nomark & \nomark & \checkmark \\
\rowcolor{white}
$\MI^\textsf{Subj}$  & \autoref{game:MIsubj}  \cite{Suri:2022b} & \nomark & \checkmark & \checkmark & \nomark & \nomark & \checkmark& \nomark & \nomark & \nomark & \checkmark & \checkmark \\

\rowcolor{white}
\addlinespace
\multicolumn{13}{c}{\textbf{Differential Privacy Distinguishability}} \\

$\DPD$               & \autoref{game:DPD}     \cite{Malek:2021,Carlini:2021b}& \nomark & \checkmark & \nomark & \checkmark & \nomark & \nomark & \checkmark & \nomark & \checkmark & \nomark & \nomark \\
$\SMI$               & \autoref{game:SMI}     \cite{Humphries:2021,Balle:2022} & \nomark & \checkmark & \nomark & \nomark & \checkmark & \nomark & \nomark & \checkmark & \checkmark & \nomark & \nomark \\

\bottomrule
\end{tabular}
\end{table*}

\section{Relations and Proofs}
\label{sec:relations}
In this section we establish relationships between privacy games.
To this end, we define a notion of \emph{reduction} and use it to translate attacks and guarantees between the five fundamental games from the previous section, or show that no generic connection can exist.

\subsection{Reductions for Privacy Games}

Inspired by notions of reduction from complexity theory and cryptography~\cite{Arora:2009}, we introduce reductions between privacy games as a means of comparing the various inference risks.
Whilst reductions in cryptography are traditionally based on asymptotic behavior governed by a security parameter, the reductions we define here are closer to those used in \emph{concrete security} proofs, in that the constants underlying the loss incurred in the reduction are made explicit.

\begin{definition}
We say that game $G_1$ is reducible to game $G_2$ if there is a constant $c > 0$ such that, for any adversary \A against $G_2$, there exists an adversary \B against $G_1$ such that
\begin{equation*}
	\Adv_{G_1}(\B) \geq c\cdot \Adv_{G_2}(\A)
\end{equation*}
We denote this using the shorthand $G_1 \reducesTo{c} G_2$ and sometimes drop the constant $c$.
\end{definition}

The intuition behind the shorthand is that game $G_1$ is at most as hard to win as $G_2$---modulo the constant $c$. This intuition holds for $c$ around or larger than 1. For $c \ll 1$, however, the lower bound on $\Adv_{G_1}(\B)$ can get close to $0$, in which case the intuition may be misleading.

\boldparagraph{Resilience to attacks}
Reductions between privacy games imply that attacks against one game translate into attacks against the other.
An equivalent reading is the contrapositive, that resilience against attacks in one game implies resilience against attacks in the other.

\begin{definition}
A game $G$ is \emph{$p$-resilient} if for all adversaries \A against $G$,
\begin{equation*}
	\Adv_{G}(\A) < p
\end{equation*}
\end{definition}

\begin{proposition}
\label{prop:resilience}
If $G_1\reducesTo{c} G_2$ and $G_1$ is $p$-resilient then $G_2$ is $p / c$-resilient.
\end{proposition}

\begin{proof}
By contradiction: If there is an attack on $G_2$ with advantage more than $p/c$, then there is one on $G_1$ with advantage more than $p$.
\end{proof}

Proofs of resilience are rare in the literature.
Prime examples are results that establish upper bounds on the advantage of a DP distinguisher when the model is trained with differential privacy~\cite{Yeom:2020,Humphries:2021}.
The tightest known bound is given in the following proposition.

\begin{proposition}[{\citet[Theorem 3.1]{Humphries:2021}}]
\label{prop:humphries}
Let $\Train$ be an $(\varepsilon,\delta)$-differentially private training algorithm. Then
\begin{equation*}
	\Adv_{\DPD}(\A) \le \frac{e^\varepsilon - 1 + 2\delta}{e^\varepsilon + 1}
\end{equation*}
\end{proposition}

Therefore, any game the DP distinguisher inference game can be reduced to (see \autoref{fig:relations_graph} for an overview) inherits the security benefits of training with differential privacy via Propositions~\ref{prop:resilience} and \ref{prop:humphries}.

\boldparagraph{Separation Results}
No reductions exist between several games. For them, we  show separation results of the form $G_1 \nreducesTo{} G_2$.
We establish such results by showing that there is an instance of $G_1$ that is resilient to attacks whereas its $G_2$ counterpart is not, and use \autoref{prop:resilience} to conclude that no reduction exists.

\subsection{Overview of Relations between Games}

\autoref{fig:relations_graph} shows the relations between the five fundamental privacy games.
Each node in the figure and in the following theorems refers to the basic game-based definition of the corresponding inference risk, \ie, \MI, \AI, \RC, \DPD, and \PI.

As expected, \PI is fully disconnected: there exists a separation result between it and every other game.
This can be attributed to the PI adversary's goal of learning properties of the training data distribution rather than about individual records as in the other games.
\RC and \DPD have the strongest threat models, where the adversary controls the entire training dataset except for one example, and hence are unsurprisingly the hardest to reduce from other games.
Finally, \MI and $\AI$ are reducible to each other and their relatively weak threat models make both \RC and \DPD reducible to them.
For this reason, we use the \MI game as the anchor for our proofs.
We next present results for a set of edges (solid lines) in \autoref{fig:relations_graph} that imply all other relations.
We defer the proofs to the Appendix.

\subsection{Reductions}
\label{sec:reductions}

Despite reductions in either direction, \MI and \AI are separable by constants in the reductions, with resilience against \AI easier to achieve than resilience against \MI.
The following theorems proved by \citet{Yeom:2020} relate \MI and \AI.

\begin{restatable}[{$\MI \reducesTo{1} \AI$~\cite[Theorem 6]{Yeom:2020}}]{thm}{MItoAI}
For any adversary $\A_\AI$ against attribute inference, there exists an adversary $\A_\MI$ against membership inference such that
\begin{equation*}
	\Adv_\MI(\A_\MI) = \Adv_\AI(\A_\AI)
\end{equation*}
\end{restatable}

\begin{restatable}[{$\AI \reducesTo{1/m} \MI$ \cite[Theorem 7]{Yeom:2020}}]{thm}{AItoMI}
Assume that for all $z \in \supp(\D)$, $\varphi(z)$ and $\pi(z)$ uniquely determine $z$.
For any adversary $\A_\MI$ against membership inference,
there exists an adversary $\A_\AI$ against attribute inference such that
\begin{equation*}
	\Adv_\AI(\A_\AI) = \frac{1}{m} \cdot \Adv_\MI(\A_\MI)
\end{equation*}
where $m$ is the number of possible values for the target attribute $\pi(z)$.
\end{restatable}

Resilience against \DPD implies resilience against all other attacks except \PI.
We present the necessary theorems below.
The remaining reductions ($\RC \reducesTo{} \AI, \DPD \reducesTo{} \AI$) are implied by the ones we show.

\citet[Theorem 3]{Balle:2022} show that training pipelines satisfying R{\'e}nyi DP (and thus $(\varepsilon,\delta)$-DP) enjoy resilience against reconstruction attacks.
In contrast, a bound on $\Adv_\DPD$ does not imply a nontrivial bound on $\varepsilon$ in $(\varepsilon,\delta)$.
In fact, $\Adv_\DPD \leq \delta$ is equivalent to $(0,\delta)$-DP.
Thus, we require an anti-concentration bound on the prior $\pi$ and that reconstruction succeeds with probability at least \nicefrac{1}{2} to reduce \DPD to \RC.

\begin{restatable}[{\hyperref[proof:DPDtoRC]{$\DPD \reducesTo{} \RC$}}]{thm}{DPDtoRC}
Let $\pi$ be a prior over samples, $S$ a dataset of $n-1$ samples, and $\ell$ a symmetric reconstruction loss satisfying the triangle inequality.
Let \A be an adversary against data reconstruction (\RC) \wrt $S$ and $\pi$ that reconstructs its challenge within error $\eta$ with probability $\gamma \geq \nicefrac{1}{2}$.
Let
\begin{equation*}
	\alpha = \inf_{z_0 \in \supp(\pi)} \Prob{z_1 \sim \pi}{\ell(z_0,z_1) > 2\eta}
\end{equation*}
There exists a DP distinguisher $\A_{\DPD \to \RC}$ such that
\begin{equation*}
	\Adv_\DPD(\A_{\DPD \to \RC}) \geq 2\alpha \left(\gamma - \frac{1}{2}\right)
\end{equation*}
\end{restatable}

DP distinguishability can be reduced to membership inference. This is an example of a generic class of reductions: In both games the adversary has the same goal and their advantage is identically defined, but in game \MI the adversary has strictly fewer capabilities than in $\DPD$.
Thus, any adversary against \MI can be turned into a valid adversary against $\DPD$ with the same advantage.
In general, a more informed/capable adversary, such as a DP distinguisher, can be used to build a reduction to games with a less informed/capable adversary.

\begin{restatable}[{\hyperref[proof:DPDtoMI]{$\DPD \reducesTo{} \MI$}}]{thm}{DPDtoMI}
For any adversary $\A_\MI$ against membership inference,
there exists a DP distinguisher $\A_\DPD$ such that
\begin{equation*}
	\Adv_\DPD(\A_\DPD) = \Adv_\MI(\A_\MI)
\end{equation*}
\end{restatable}

Finally, we show that a membership inference attack can be turned into a reconstruction attack, with a constant depending on the size of the support of the training data distribution.

\begin{restatable}[{\hyperref[proof:RCtoMI]{$\RC \reducesTo{1/|\supp(\D)|} \MI$}}]{thm}{RCtoMI}
For any membership inference adversary \A against $\MI(\Train,\D,n)$ there exists a reconstruction adversary \B against $\RC^{\textsf{Ran}}(\D,n,\D,\Train)$ (\ie, with prior $\pi = \D$) such that
\begin{equation*}
	\Adv_{\RC^\textsf{Ran}}(\B) = \frac{1}{|\supp(\D)|}\cdot \Adv_\MI(\A)
\end{equation*}
\end{restatable}

\subsection{Separation Results}
\label{sec:separation}

\begin{restatable}[{\hyperref[proof:MItoPI]{$\MI \nreducesTo \PI$}}]{thm}{MItoPI}
Resilience against membership inference does not imply resilience against property inference.
\end{restatable}

\begin{restatable}[{\hyperref[proof:MItoDPD]{$\MI \nreducesTo \DPD$}}]{thm}{MItoDPD}
Resilience against membership inference does not imply resilience against DP distinguishability.
\end{restatable}

\begin{restatable}[{\hyperref[proof:MItoRC]{$\MI \nreducesTo \RC$}}]{thm}{MItoRC}
Resilience against membership inference does not imply resilience against reconstruction.
\end{restatable}

This last counterintuitive separation result stems from a discrepancy between adversary capabilities: The \MI game is based on an average case scenario, while the reconstruction game assumes a more informed worst-case adversary.
By considering a membership adversary matching the capabilities of the adversary in the \RC game, we can build a reduction to data reconstruction.
We show this in \autoref{thmt@@SMItoRC} in the Appendix, which reduces the strong membership inference game \SMI (\autoref{game:DPD}) to game \RC.

\begin{restatable}[{\hyperref[proof:PItoMI]{$\PI \nreducesTo \MI$}}]{thm}{PItoMI}
Resilience against property inference does not imply resilience against membership inference.
\end{restatable}

\begin{restatable}[{\hyperref[proof:RCtoDPD]{$\RC \nreducesTo \DPD$}}]{thm}{RCtoDPD}
Resilience against reconstruction does not imply resilience against DP distinguishability.
\end{restatable}

\begin{restatable}[{\hyperref[proof:DPDtoPI]{$\DPD \nreducesTo \PI$}}]{thm}{DPDtoPI}
Resilience against DP distinguishability does not imply resilience against property inference.
\end{restatable}

\section{Case Study: Mixture Model Membership Inference}
\label{sec:case_study}
We present a case study where we showcase the expressive power and rigor of privacy games.
In particular, we show that a novel variant of membership inference can be decomposed into a combination of membership and property inference.
This complex relationship goes beyond the direct reductions presented in \autoref{sec:relations}.
In our proofs, we exploit \emph{code-based} reductions structured as a sequence of games; \ie, our arguments rely on transforming code with a formal semantics.

The game we target is due to \citet{Humphries:2021}, who use it to model membership inference attacks in the presence of dependencies in the training data.
In their game (\MM in \autoref{game:MM}), the training data follows a two-stage \emph{mixture model}.
Examples in the training dataset and the target example are chosen independently from two data distributions, $\D_k$ and $\D_{k'}$, which are chosen uniformly at random without replacement from $K$ possible distributions $\DD = \{\D_1, \dotsc, \D_K\}$.

\begin{game}
\caption[F]{\boxA{$\MM_{\phantom{}}$} \boxB{$G_0$}}
\label{game:MM}
  \KwIn{$\Train, \DD, n, \A$}
  $k  \sim [K]$\;
  $k' \sim [K] \setminus \{k\}$\;
  $S \sim \D_k^n$\;
  $\theta \gets \Train(S)$\;
  $b \sim \bit$\;
  \eIf{$b = 0$}{
    \boxA{$z \sim S_{\phantom{}}$} \boxB{$z \sim \D_k$}
  }
  {
    $z \sim \D_{k'}$
  }
  $\guess{b} \gets \A(\Train, \DD, n, \theta, z)$
\end{game}
\vspace{-15pt}
\begin{game}
\caption[F]{$G_1$}
\label{game:G1}
  \KwIn{$\Train, \DD, n, \A$}
  $k  \sim [K]$\;
  $k' \sim [K] \setminus \{k\}$\;
  $z \sim \D_k$\;
  $b \sim \bit$\;
  \eIf{$b = 0$}{
    $S \sim \D_k^n$
  }
  {
    $S \sim \D_{k'}^n$
  }
  $\theta \gets \Train(S)$\;
  $\guess{b} \gets \A(\Train, \DD, n, \theta, z)$
\end{game}

We show that \MM can be decomposed into a property inference goal (inferring the training data distribution) and a membership inference goal (inferring whether a target example has been sampled from the training data distribution $\D_k$ or from the training dataset $S$).

\begin{restatable}{thm}{relationMMandMIandPI}
For any adversary \A against \MM, there exist adversaries $\A_\MI^i$ and $\A_\PI^{i,j}$ such that
\begin{equation*}
  \Adv_{\MM}(\A) \leq
  \max_{i \in [K]} \Adv_{\MI_i}(\A_\MI^i) + \max_{i \neq j \in [K]} \Adv_{\PI_{i,j}}(\A_\PI^{i,j})
\end{equation*}
where $\MI_i$ is the membership inference game with training data distribution $\D_i$, and in $\PI_{i,j}$ the property to infer is whether the training data distribution is $\D_i$ or $\D_j$.
\end{restatable}

\begin{proof}
Let \A be an adversary against \MM.
Consider $G_0$ shown alongside \MM in \autoref{game:MM}.
Its only difference \wrt \MM is that when $b = 0$, the example $z$ is freshly sampled from the training data distribution $\D_k$ rather than from the training dataset $S$.
Conditioned on $b = 0, k = i$, distinguishing between games $G_0$ and \MM is as difficult as winning a membership inference game.
We show this using a black-box reduction: fixing $k = i$, we construct an adversary $\A_\MI^i$ that uses \A as an oracle to guess the challenge bit $b$ in game $\MI_i$ (see \autoref{game:MI_i}).
$\A_\MI^i$ simply forwards its inputs $\Train, n, \theta, z$ to \A, passing to it in addition the distribution set $\DD$.

\begin{game}
\caption[F]{$\MI_i$}
\label{game:MI_i}
  \KwIn{$\Train, \D_i, n, \A$}
    $S \sim \D_i^n$\;
    $\theta \gets \Train(S)$\;
    $b \sim \bit$\;
    \leIf{$b = 0$}{$z \sim S$}{$z \sim \D_i$}
    $\guess{b} \gets \A_\MI^i(\Train, \D_i, n, \theta, z)$
\end{game}
\vspace{-15pt}
\begin{adversary}
\caption[F]{$\A_\MI^i$}
  \KwIn{$\Train, \D_i, n, \theta, z$}
    \Return $\A(\Train, \DD, n, \theta, z)$
\end{adversary}

Game \MM conditioned on $b = 0, k = i$ is equivalent to $\MI_i$ conditioned on $b = 0$.
Likewise, game $G_0$ conditioned on $b = 0,k = i$ is equivalent to $\MI_i$ conditioned on $b = 1$. Hence,
\begin{align}
\Adv_{\MI_i}(\A_\MI^i) &=
\Prob{\MI_i}{\lnot\guess{b} \mid \lnot b} -
\Prob{\MI_i}{\lnot\guess{b} \mid b}
\nonumber \\
& \hspace{-1.6cm} =
\Prob{\MM}{\lnot\guess{b} \mid \lnot b, k = i} -
\Prob{G_0}{\lnot\guess{b} \mid \lnot b, k = i}
\label{eq:MM_G0}
\end{align}
Game \MM conditioned on $b = 1$ is equivalent to $G_0$ conditioned on $b = 1$, and so we have
\begin{align}
\Adv_{\MM}(\A) &=
\Prob{\MM}{\lnot\guess{b} \mid \lnot b} -
\Prob{\MM}{\lnot\guess{b} \mid b} \nonumber \\
& \hspace{-1.6cm} =\!
\frac{1}{K}\!
  \sum_{i=1}^K
  \Prob{\MM}{\lnot\guess{b} \mid \lnot b, k = i} - \Prob{\MM}{\lnot\guess{b} \mid b, k = i}
\nonumber \\
& \hspace{-1.6cm} =\!
\frac{1}{K}\!
  \sum_{i=1}^K
    \Adv_{\MI_i}(\A_\MI^i) \!+\!
    \Prob{G_0}{\lnot\guess{b} \!\mid\! \lnot b} \!\!-\!\!
    \Prob{G_0}{\lnot\guess{b} \!\mid\! b}
\label{eq:MMbound}
\end{align}
where the last equation follows from \eqref{eq:MM_G0} and the fact that $b$ and $k$ are independent.

We reformulate $G_0$ as $G_1$ (see \autoref{game:G1}). To see why both formulations are equivalent, note that conditioned on $b = 0$, in both games $S$ and $z$ are sampled from the same distribution chosen uniformly from $\DD$, while conditioned on $b = 1$, $S$ and $z$ are sampled each from one of two distributions sampled without replacement from $\DD$.
Since $b$ is independently sampled in the same way, both games result in the same joint distribution of $\theta,z,b$, and therefore $\guess{b},b$:
\begin{align}
\Prob{G_0}{\lnot\guess{b} \mid \lnot b} & =
\Prob{G_1}{\lnot\guess{b} \mid \lnot b} \\
\Prob{G_0}{\lnot\guess{b} \mid b} & =
\Prob{G_1}{\lnot\guess{b} \mid b}
\label{eq:G0_G1}
\end{align}
Next, we show using a black-box reduction that distinguishing between the case when $b = 0$ and $b = 1$ in $G_1$ conditioned on $k = i, k' = j$ is as hard as guessing the challenge bit in the property inference experiment $\PI_{i,j}$ shown in \autoref{game:PI_i_j}.
To do this, we construct an adversary $\A_\PI^{i,j}$ that uses \A as a black-box. $\A_\PI^{i,j}$ perfectly simulates the inputs to \A in $G_1$ by forwarding its own inputs and freshly sampling $z$ from $\D_i$.
\begin{align*}
\Adv_{\PI_{i,j}}(\A_\PI^{i,j}) &=
  \Prob{\PI_{i,j}}{\lnot\guess{b} \mid \lnot b} -
  \Prob{\PI_{i,j}}{\lnot\guess{b} \mid b} \nonumber \\
& \hspace{-1.9cm} =
  \Prob{G_1}{\lnot\guess{b} \!\mid\! \lnot b, k \!=\! i, k' \!=\! j} \!-\!
  \Prob{G_1}{\lnot\guess{b} \!\mid\! b, k \!=\! i, k' \!=\! j}
\label{eq:PI}
\end{align*}
Putting this and \eqref{eq:MMbound}--\eqref{eq:G0_G1} together we obtain
\begin{align*}
\Adv_{\MM}(\A) &= 
\begin{aligned}[t]
  & \frac{1}{K}
  \sum_{i=1}^K \Adv_{\MI_i}(\A_\MI^i) \\
  &+ \frac{1}{K(K - 1)} \sum_{i=1}^K \sum_{j=1,i\neq j}^K \Adv_{\PI_{i,j}}(\A_\PI^{i,j})
\end{aligned}\\
& \leq \max_{i \in [K]} \Adv_{\MI_i}(\A_\MI^i) + \max_{i \neq j \in [K]} \Adv_{\PI_{i,j}}(\A_\PI^{i,j})
\qedhere
\end{align*}

\begin{game}
\caption[F]{$\PI_{i,j}$}
\label{game:PI_i_j}
  \KwIn{$\Train, \D_i, \D_j, n, \A$}
  $b \sim \bit$\;
  \eIf{$b = 0$}{
    $S \sim \D_i^n$
  }
  {
    $S \sim \D_j^n$
  }
  $\theta \gets \Train(S)$\;
  $\guess{b} \gets \A_\PI^{i,j}(\Train, \D_i, \D_j, n, \theta)$
\end{game}
\vspace{-15pt}
\begin{adversary}
\caption[F]{$\A_\PI^{i,j}$}
  \KwIn{$\Train, \D_i, \D_j, n, \theta$}
  $z \sim \D_i$\;
  \Return $\A(\Train, \DD, n, \theta, z)$
\end{adversary}
\end{proof}

\section{Discussion}
\label{sec:discussion}

We discuss strategies for choosing privacy games, their current and future uses, and their limitations.

\subsection{Selecting Games}

With the variety of privacy games in the literature, it is natural to ask whether there is a \emph{canonical} game that should be used instead of others.
We believe this is not the case, \ie, no single game is the best choice in all circumstances because subtle differences in threat scenarios can lead to vastly different privacy evaluations (see, \eg,~\cite{Carlini:2021b}).
Instead, we recommend that users of games leverage the building blocks we provide in this paper to design games that accurately capture their application-specific threat models.
For a given threat model, however, some differences between modelling choices are less important (\eg, in MI, whether one samples nonmembers from the full distribution or excluding the training set), and we highlight this distinction throughout the paper.

\subsection{Current Uses of Privacy Games}

The use of privacy games has become prevalent in the literature on machine learning privacy.
As of today, there have been two main applications:
\begin{inparaenum}
\item supporting the \emph{empirical evaluation} of machine learning systems against a variety of threats, and
\item \emph{comparing} the strength of privacy properties and attacks.
\end{inparaenum}
Reductions enable a third application: translating \emph{provable guarantees} from one property to another.

Game-based definitions of inference risks are presented in often inconsistent form fragmented across the literature and only a few of the reductions and separation results in~\autoref{fig:relations_graph} were made explicit.
We present a common vocabulary for game-based definitions, formalize games for five fundamental inference risks, and establish connections between them.

\subsection{Prospective Uses of Privacy Games}

We highlight two other promising uses for games.

\boldparagraph{Communicating privacy properties}
Reasoning about ML privacy risks is not the exclusive purview of researchers.
Other personas, \eg, privacy managers and auditors, need to make decisions about the compliance of training pipelines with regulatory or contractual constraints.
Based on our experience, privacy managers currently base their reasoning on (1) empirical privacy evaluations, (2)
formal guarantees of mechanisms such as DP-SGD, and (3) informal texts such as the \emph{Opinion 05/2014}~\cite{article29} of the European Commission's Article 29 Working Party.
They are then faced with the daunting task of combining these pieces into a coherent picture to assess the privacy risks of specific applications.
Privacy games can help with this task: by making the threat model and assumptions about dataset creation and training explicit, they can disambiguate interpretations and can abstract an application scenario with respect to its (provable and empirical) privacy properties.
Indeed, based on our initial experience, games facilitate discussing privacy goals and guarantees with stakeholders making guidelines and decisions around ML privacy.

\boldparagraph{Mechanization of proofs}
An advantage of the game-based formalism is that games can be given an unambiguous semantics as probabilistic programs.
This enables reasoning about games using program logics and manipulating them using program transformations.
Reusable program transformations (\eg, procedure inlining) and proof techniques (\eg, conditioning on events) arise naturally and make proofs more amenable.
As we show in \autoref{sec:case_study}, our proofs exhibit some of these patterns.

We envisage techniques and frameworks to reason about game-based cryptographic proofs (\eg, EasyCrypt, FCF) being repurposed to reason about privacy games.
The apparent complexity of privacy games compared to cryptographic games is not an obstacle since most proofs manipulate training algorithms, models, and data as abstract objects with minimal structure.
For instance, we think it is possible to formalize the proof in \autoref{sec:case_study} in a tool like EasyCrypt.
The main challenge for mechanizing proofs about privacy games is that, unlike cryptographic games, privacy games sometimes require reasoning about continuous distributions (\eg, Gaussian noise in DP-SGD), but logics implemented in existing frameworks often assume a discrete probability space.

\subsection{Limitations of Privacy Games}

Privacy games are sequential probabilistic programs; they are not an immediate fit for expressing concurrent computations.
This prevents the direct application of games to important scenarios such as federated learning (FL). Intuitively, this is due to the hardness of modeling the various possible parallel interactions between the different parties.
The situation is similar for cryptographic games, where process calculi are used instead of games for modeling more complex multi-party interactions~\cite{Blanchet:2007,Meier:2013}.
It is an open question whether these calculi could also be used in the context of concurrent ML scenarios such as FL.

\section{Related Work}
\label{sec:related}
\boldparagraph{Alternatives}
We discuss below informal and formal alternatives to games to express privacy properties.

A key example of a \emph{formal} property is Differential Privacy~\cite{Dwork:2006}. The definition of Differential privacy is relational, in that it compares the probability of events in two alternative worlds. DP abstracts from many details that are relevant for threat modelling, such as adversary capabilities, goals, and background knowledge, as well as the way datasets are created. This has led to disagreements in the literature about the consequences of differential privacy (see~\cite{Tschantz:2020}).

A key example of an \emph{informal} account of privacy properties is the \emph{Opinion 05/2014 on Anonymization Techniques}~\cite{article29} that complements the EU General Data Protection Regulation (GDPR) with practical recommendations for the use of anonymization techniques to meet the requirements set out by the regulator.
In this influential document, the authors identify three privacy risks: \emph{singling out}, \emph{linkability}, and \emph{inference}. They analyze the suitability of different anonymization techniques---including $k$-anonymity and DP---for mitigating these risks, but the discussion remains inconclusive due to the lack of precise definitions. Subsequent research~\cite{Nissim:2020} rigorously revisited the notion of singling out and suggested reconsidering the Opinion recommendations.

Game-based definitions address shortcomings of both alternatives: They make the threat model and assumptions explicit and precise, which helps disambiguate interpretations.

\boldparagraph{Game-based privacy proofs}
\citet{Nissim:2018} construct a privacy game that reflects the requirements of the U.S.\xspace Family Educational Rights and Privacy Act (FERPA) for protecting privacy in releases of education records, and show in a proof structured as a sequence of games that DP is enough to satisfy these requirements.
While constructing the game, they identify dimensions similar to our anatomy in \autoref{sec:anatomy}.

\boldparagraph{Surveys and taxonomies on privacy}
Several papers propose taxonomies of privacy attacks against machine learning systems~\cite{Liu:2021c,Rigaki:2020,DeCristofaro:2021}.
\citet{Papernot:2018} focus on systematizing the possible attack surfaces of standard machine learning pipelines.
\citet{Desfontaines:2020} systematically study variants and extensions of differential privacy.
Before attacks against ML systems were demonstrated, \citet{Li:2013} proposed a unifying framework for membership and differential privacy definitions mainly applicable to database systems.

\section*{Acknowledgment}
\noindent
This work was partially funded by the U.S. National Science Foundation through the Center for Trustworthy Machine Learning (\#1804603).

\appendix

\subsection{Direct, Single-Query Membership Inference}
\label{sec:SQMI}
\citet{Tang:2022} present single-query variants of membership inference  where the adversary is given only the model output on the challenge point. 
In their base game (\autoref{game:SQMI}), the adversary selects a universe of $2n$ points from where $n$ points are sub-sampled to construct the training dataset of the target model.
The adversary goal is to infer whether a challenge $z_j$ uniformly sampled from the initial $2n$ points was used to train the model, \ie, guess $B[j]$, given just the model output on $z_j$.
They also consider variants where the set of $2n$ points is fixed externally, and a worst-case variant where the challenge $z_j$ is selected by the adversary.

\begin{game}[htpb]
\caption[F]{$\MI^{\textsf{SQ}}$}
\label{game:SQMI}
    \KwIn{$\Train, n, \A, \A^\prime$}
    $\{z_i\}_{i \in [2n]} \gets \A^\prime(\Train, n)$\;
    $B \sim \bit^{2n} \textrm{~s.t.~} \sum_{i \in [2n]} B[i] = n$\;
    $S \gets \{z_i\, |\, B[i] = 0\}_{i \in [2n]}$\;	
    $\theta \gets \Train(S)$\;
    $j \sim [2n]$\;
    $\guess{b} \gets \A(\Train, n, \{z_i\}_{i \in [2n]}, j, \theta(z_j))$
\end{game}

\subsection{Deferred Proofs}
\label{sec:proofs}
\begin{restatable}[$\SMI \reducesTo{} \RC$]{thm}{SMItoRC}
Let $z_0,z_1$ be two samples, $S$ a dataset of $n-1$ samples, and $\ell$ a symmetric reconstruction loss satisfying the triangle inequality.
Let $\A$ be an adversary against data reconstruction (\RC) \wrt $S$ and the uniform prior on $\{z_0,z_1\}$ that reconstructs its challenge with error $\eta < \ell(z_0,z_1)/2$ with probability $\gamma$.
Then, there exists a strong membership inference adversary $\A_{\SMI \to \RC}$ such that
\begin{equation*}
  \Adv_\SMI(\A_{\SMI \to \RC}) \geq 2\gamma - 1
\end{equation*}
\end{restatable}

\begin{proof}
\begin{adversary}[htpb]
\caption[F]{$\A_{\SMI \to \RC}$}
\label{adv:SMI_to_RC}
  \KwIn{$\Train, \theta, S, z_0, z_1$}
  $\guess{z} \gets \A(\Train, \theta, S)$\;
  \eIf{$\ell(z_0, \guess{z}) < \ell(z_1, \guess{z})$}{
    \Return $0$
  }
  {
    \Return $1$
  }
\end{adversary}

Define $\A_{\SMI \to \RC}$ as in Adversary~\ref{adv:SMI_to_RC}.
For any $\guess{z}$, we have from the triangle inequality,
\begin{align*}
\ell(z_0, \guess{z}) &< \ell(z_0, z_1) / 2 < (\ell(z_0, \guess{z}) + \ell(\guess{z}, z_1)) / 2 \implies \\
\ell(z_0, \guess{z}) &< \ell(z_1, \guess{z})
\end{align*}
Therefore, when $b = 0$ in \SMI and \A succeeds in reconstructing $z_0$ within error $\eta$, $\A_{\SMI \to \RC}$ guesses correctly.
Similarly, when $b = 1$ and \A succeeds in reconstructing $z_1$ within error $\eta$, $\A_{\SMI \to \RC}$ guesses correctly.
Thus, $\A_{\SMI \to \RC}$ guesses $b$ correctly at least with probability $\gamma$ and
\begin{equation*}
\Adv_\SMI(\A_{\SMI \to \RC}) =
  2\!\Prob{\SMI(\cdots)}{\guess{b} = b} \!-\! 1
  \geq 2\gamma \!- 1 \qedhere
\end{equation*}
\end{proof}

\DPDtoRC*

\begin{proof}
\label{proof:DPDtoRC}
Observe that $1 - \alpha$ is the baseline success of a reconstruction adversary with error $2\eta$ (see~\autoref{eq:RC_baseline}).

Define $\A'_{\DPD \to \RC}$ as in Adversary~\ref{adv:DPD_to_RC_choose} and $\A_{\DPD \to \RC}$ as in Adversary~\ref{adv:DPD_to_RC}.

\begin{adversary}[htpb]
\caption[F]{$\A'_{\DPD \to \RC}$}
\label{adv:DPD_to_RC_choose}
  \KwIn{$\Train, n$}
  $z_0,z_1 \sim \pi$\;
  \Return $S, z_0, z_1$
\end{adversary}
\vspace{-10pt}
\begin{adversary}[htpb]
\caption[F]{$\A_{\DPD \to \RC}$}
\label{adv:DPD_to_RC}
  \KwIn{$\Train, \theta, S, z_0, z_1$}
  \eIf{$\ell(z_0,z_1) \leq 2\eta$}{
    $\guess{b} \sim \bit$\;
  }
  {
    $\guess{b} \gets \A_{\SMI \to \RC}(\Train, \theta, S, z_0, z_1)$
  }
  \Return $\guess{b}$
\end{adversary}

In the \DPD game, when $\ell(z_0,z_1) > 2\eta$, which occurs with probability at least $\alpha$, a similar analysis as in \autoref{thmt@@SMItoRC} shows that $\A_{\DPD \to \RC}$ guesses $b$ correctly whenever $\A$ succeeds in reconstructing its challenge within error $\eta$. Otherwise, the adversary guesses with probability $1/2$. Thus,
$$
\begin{array}{@{}l@{~}l}
\Prob{\DPD}{\guess{b} = b} &\geq \Prob{\DPD}{\guess{b} = b | \ell(z_0,z_1) > 2\eta} \alpha\ + \\
\quad
      & \quad \Prob{\DPD}{\guess{b} = b | \ell(z_0,z_1) \leq 2\eta} (1-\alpha) \\
& = \gamma \alpha + \frac{1}{2} (1 - \alpha)
\end{array}
$$
The DPD advantage of $\A_{\DPD \to \RC}$ is
\begin{align*}
  \Adv_{\DPD}(\A_{\DPD \to \RC}) &= 2\Prob{\DPD}{\guess{b} = b} - 1 \\
  &\geq 2\alpha \left(\gamma - \frac{1}{2}\right) \; \qedhere
\end{align*}
\end{proof}

\DPDtoMI*

\begin{proof}
\label{proof:DPDtoMI}

Let \A be an adversary against $\MI(\Train,\D,n)$.
We construct an adversary against $\DPD(\Train,n)$ as in Adversary~\ref{adv:DPD_to_MI_choose} and \ref{adv:DPD_to_MI}.
These adversary procedures, when inlined in $\DPD(\Train,n)$ (\autoref{game:DPD}), result in an experiment semantically equivalent to $\MI(\Train,\D,n,\A)$ (\autoref{game:MI}).
Thus,
\begin{equation*}
  \Adv_\DPD(\A_{\DPD \to \MI}) = \Adv_\MI(\A) \qedhere
\end{equation*}

\begin{adversary}[htpb]
\caption[F]{$\A'_{\DPD \to \MI}$}
\label{adv:DPD_to_MI_choose}
  \KwIn{$\Train, n$}
  $S \sim \D^{n-1}$\;
  $z_0, z_1 \sim \D$\;
  \Return $S, z_0, z_1$
\end{adversary}
\vspace{-10pt}
\begin{adversary}[htpb]
\caption[F]{$\A_{\DPD \to \MI}$}
\label{adv:DPD_to_MI}
  \KwIn{$\Train, \theta, S, z_0, z_1$}
  $\guess{b} \gets \A(\Train, \D, n, \theta, z_0)$\;
  \Return $\guess{b}$
\end{adversary}
\end{proof}

\RCtoMI*

\begin{proof}
\label{proof:RCtoMI}
Consider \autoref{game:AI_informed}, which is equivalent to \AI except the adversary is also given $S$.

\begin{game}
\caption[F]{$\AI'$}
\label{game:AI_informed}
  \KwIn{$\Train, \D, n, \B, \varphi, \pi$}
  $S \sim \D^{n-1}$\;
  $z_0, z_1 \sim \D$\;
  $b \sim \bit$\;
  $\theta \gets \Train(S \cup \{z_b\})$\;
  $\guess{z} \gets \B(\Train, \D, n, \theta, S, \varphi(z_0))$
\end{game}

The reconstruction advantage of \B coincides with its advantage
in $\AI'$ in the special case where $\varphi(z) = \bot$ and $\pi(z) = z$, \ie, the adversary has to reconstruct all attributes.
This is because $\varphi(z_0) = \bot$ and thus the guess $\guess{z}$ is independent of $z_0$ conditioned on $b = 1$.
\begin{equation*}
  \Adv_{\RC^\textsf{Ran}}(\B) =
    \Prob{\AI'}{\guess{z} = z_0 | b = 0} -
    \Prob{\AI'}{\guess{z} = z_0 | b = 1}
\end{equation*}

The rest of the proof is similar to the proof of \autoref{thmt@@AItoMI}, but we present it for the sake of completeness.

Let \A be an adversary against $\MI(\Train,\D,n)$.
We construct an adversary \B against $\RC^\textsf{Ran}(\D,n,\D,\Train)$, shown in \autoref{adv:RC_to_MI}, which uses \A to reconstruct its challenge.

\begin{adversary}
\caption[F]{\B}
\label{adv:RC_to_MI}
  \KwIn{$\Train, \theta, S$}
  $z' \sim \supp(\D)$\;
  $\guess{b} \gets \A(\Train, \D, n, \theta)$\;
  \eIf{$\guess{b} = 0$}{
    \Return $z'$
  }
  {
    \Return $\bot$
  }
\end{adversary}

Denote $\D(z_i)$ the quantity $\Prob{z \sim \D}{z = z_i}$, \ie,
the probability mass of \D at $z_i$ and let $m = |\supp(\D)|$.
In the following, we use \RC to denote the game $\RC^\textsf{Ran}(\D,n,\D,\Train,\B)$ and \MI to denote $\MI(\Train,\D,n,\A)$.

Since \B guesses $\guess{z} = z$ if and only if $z' = z$ and $\guess{b} = 0$, for any $z_i \in \supp(\D)$ we have for $\hat{b} \in \bit$

\begin{equation}
\!\!
\Prob{\AI'}{\guess{z} \!=\! z | b \!=\! \hat{b}, z \!=\! z_i} \!=\!
\frac{1}{m}\! \Prob{\AI'}{\guess{b} \!=\! 0 | b \!=\! \hat{b}, z \!=\! z_i}
\end{equation}

Hence, the advantage of $\B$ is

\begin{equation*}
\begin{split}
& \Adv_{\RC^\textsf{Ran}}(\B) \\
& =\!\!\!\! \sum_{z_i \in \supp(\D)}
  \!\!\!\!\D(z_i)
    \left( \Prob{\AI'}{\guess{z} = z_0 | b = 0, z = z_i} \right. \\
&\qquad\qquad\qquad\quad \left. - \Prob{\AI'}{\guess{z} = z_0 | b = 1, z = z_i} \right) \\
& = \frac{1}{m} \!\!\!\!
  \sum_{z_i \in \supp(\D)}
  \!\!\!\!\D(z_i)
    \left( \Prob{\AI'}{\guess{b} = 0 | b = 0, z = z_i} \right. \\
&\qquad\qquad\qquad\qquad \left. - \Prob{\AI'}{\guess{b} = 0 | b = 1, z = z_i} \right) \\
& = \frac{1}{m}
    \left( \Prob{\AI'}{\guess{b} = 0 | b = 0} -
           \Prob{\AI'}{\guess{b} = 0 | b = 1} \right) \\
& = \frac{1}{m} \Adv_\MI(\A)
\end{split}
\end{equation*}

The penultimate equality holds because $b$ and $z$ are independent.
The last equality holds because game $\AI'(\Train,\D,n,\B)$ matches game $\MI(\Train,\D,n,\A)$ and so the joint distribution of $\guess{b},b$ is the identical in both games. \qedhere

\end{proof}

\MItoDPD*

\begin{proof}
\label{proof:MItoDPD}
We show that there are training pipelines that are arbitrarily resilient against membership inference attacks but completely insecure against DP distinguishing attacks.

We construct a training pipeline $(\Train, \D, n)$ such that the MI advantage of an adversary against it is at most $\nicefrac{1}{\sqrt{n}}$, and so vanishes as $n$ grows.
Yet, we exhibit a DP distinguisher against the pipeline that achieves perfect advantage.

\begin{game}
\caption{$\MI'$}
\label{game:MI_informed}
    \KwIn{$\Train, \D, n, \A$}
    $b \sim \bit$\;
    $S \sim \D^{n-1}$\;
    $z_0, z_1 \sim \D$\;
    $\theta \gets \Train(S \cup \{z_b\})$\;
    $\guess{b} \gets \A(\Train, \D, n, \theta, z_0, z_1)$
\end{game}

Let $\D = \mathrm{Bernoulli}(p)$ and $\Train(S) = \sum_{x \in S} x$.
Consider \autoref{game:MI_informed}.
If the adversary were only given $z_0$, this game would be equivalent to the basic \MI game (\autoref{game:MI_example}).
Since the adversary is given strictly more information, any bound on its advantage in this game would also bound the MI advantage of adversaries against the training pipeline.
The adversary must distinguish between two simple hypotheses:
\begin{itemize}
\item $H_0: \theta \sim \mathrm{Binomial}(n-1, p) + z_0$
\item $H_1: \theta \sim \mathrm{Binomial}(n-1, p) + z_1$
\end{itemize}
When $z_0 = z_1$, these coincide and the advantage of the adversary is 0.
Otherwise, without loss of generality, assume $z_b = b$.
By the Neyman-Pearson lemma, a likelihood ratio test yields the most powerful test for a significance $\alpha$ (\ie, Type-I error, false positive rate).
Let $f$ and $F$ be the probability mass and cumulative distribution function of $\mathrm{Binomial}(n-1, p)$, respectively.
The likelihood ratio is
\begin{equation*}
\Lambda(\theta = k) =
  \begin{cases}
    \infty & \text{if } k = 0 \\
    0      & \text{if } k = n \\
    \frac{f(k)}{f(k-1)} = \frac{(n - k) p}{k (1-p)} & \text{otherwise}
  \end{cases}
\end{equation*}
The test rejects $H_0$ when $\Lambda(\theta) < c$, for some $c$.
The false positive rate $\alpha$ (the probability of rejecting $H_0$ when $H_0$ is true) is
\begin{align*}
\Pr_{H_0} (\Lambda(\theta) < c)
    &= \Pr_{H_0}\left( \frac{(n - k) p}{k (1 - p)} < c \right)  \\
    &= \Pr_{H_0} \left( k > \frac{n p}{p + c (1-p)} \right) \\
    &= 1 - F \left( \frac{n p}{p + c (1-p)} \right)
\end{align*}
The false negative rate $\beta$ is
\begin{align*}
\Pr_{H_1} (\Lambda(\theta) \geq c)
    &= \Pr_{H_1}\left( \frac{(n - k) p}{k (1 - p)} \geq c \right)  \\
    &= \Pr_{H_1} \left( k \leq \frac{n p}{p + c (1-p)} - 1 \right) \\
    &= F \left( \frac{n p}{p + c (1-p)} - 1 \right)
\end{align*}
Now, take $p = 0.5$ and assume that $n \geq 4$ and that $n$ is even so that the mode of $\mathrm{Binomial}(n-1, p)$ is $n/2$.
The MI advantage of the adversary is
\begin{align*}
\Adv_\MI(\A)
    &= \frac{1}{2} ( f(0) + f(n-1) + (1 - \alpha - \beta) ) \\
    &= f(0) + \frac{1}{2} f\left( \frac{n p}{p + c (1-p)} \right) \\
    &\leq \frac{1}{2^{n-1}} + \frac{f(n/2)}{2} \\
    &\leq \frac{1}{2\sqrt{n}} + \frac{1}{2\sqrt{n}} = \frac{1}{\sqrt{n}}
\end{align*}

On the other hand, a DP distinguisher \A that chooses $z_0 = 0, z_1 = 1$, an arbitrary $S$, and that guesses $\guess{b} = \theta - \sum_{x\in S} S$, has perfect advantage $\Adv_\DPD(\A) = 1$.
\end{proof}

\balance

\MItoPI*
\label{proof:MItoPI}

\begin{proof}
We construct a training pipeline $(\Train, \D_b, n)$ that is arbitrarily resilient to membership inference for $b \in \bit$.
Yet, we exhibit a property inference attack against it that achieves perfect advantage.

Let $\D_b = \mathrm{Bernoulli}(p_b)$ with $p_0 \neq p_1$ and $\Train(S) = \sum_{x \in S}x$.
As shown in \autoref{thmt@@MItoDPD}, the advantage of a membership inference adversary against $(\Train, \D_b, n)$ is at most $\nicefrac{1}{\sqrt{n}}$.
However, as $n$ grows, $\Train(S)/n$ is an unbiased estimator for the mean $p_b$, which allows a property inference adversary to easily distinguish between $\D_0$ and $\D_1$, particularly when $p_0$ and $p_1$ are far apart.
\end{proof}

\MItoRC*

\begin{proof}
\label{proof:MItoRC}
It suffices to show that the training pipeline from \autoref{thmt@@MItoDPD}, which is resilient to membership inference attacks, admits a reconstruction attack.
For this, recall that in \RC the adversary knows the dataset $S$ (but not the target sample $z$).
For the pipeline $(\Train, \D, n)$ given in \autoref{thmt@@MItoDPD}, $z$ can be perfectly reconstructed since $z = \theta - \sum_{x\in S} x$.
\end{proof}

\PItoMI*

\begin{proof}
\label{proof:PItoMI}
We exhibit a pipeline resilient to property inference that is completely vulnerable to a membership inference attack.

Let \D be an arbitrary distribution and define $\D_b$ so that
$z \sim \D_b \equiv x \sim \D; z \gets (x, b)$.
Let $n > 0$ and define
\begin{equation*}
  \Train(S) = \{ x \in S | (x,y) \in S \}
\end{equation*}
A membership inference adversary against $\MI(\Train, \D, n)$ that given $\theta, z_0 = (x,y)$ returns 0 if and only if $x \in S$ achieves the maximum advantage, \ie,
\begin{equation*}
  1 - \Prob{S \sim \D^n; x \sim \D}{x \in S}
\end{equation*}
However, a property inference adversary gets no information about $b$ as $\theta$ and $b$ are independent, so its advantage is 0.
\end{proof}

\RCtoDPD*

\begin{proof}
\label{proof:RCtoDPD}
\citet[Theorem 5]{Balle:2022} show that resilience against reconstruction \wrt all priors in a family of distributions concentrated on all ordered pairs of distinct examples implies $(\varepsilon,\delta)$-DP, and hence via \autoref{prop:humphries} resilience against \DPD.

However, resilience against a single prior $\pi$, even if its support includes all possible examples, is clearly insufficient to guarantee resilience against \DPD.
To see why, consider a deterministic \DPD adversary that picks $S, z_0, z_1$. Given error bound $\eta$ and success probability $\gamma$, all reconstruction adversaries can have error larger than $\eta$ when $z \notin \{z_0,z_1\}$ but reconstruct $z \in \{z_0,z_1\}$ perfectly, as long as $\Prob{z \sim \pi}{z \in \{z_0,z_1\}} < \gamma$, \ie, the probability mass of the prior on $\{z_0,z_1\}$.
The situation is worse when $z_0,z_1 \notin \supp(\pi)$, where resilience against reconstruction for arbitrary $\eta,\gamma$ is compatible with perfect \DPD advantage.
\end{proof}

\DPDtoPI*

\begin{proof}
\label{proof:DPDtoPI}
Let $\varepsilon, \delta \in (0,1)$.
We build two training pipelines $(\Train, \D_b, n)$, $b \in \bit$, that satisfy ($\varepsilon, \delta$)-DP and thus
are resilient against \DPD (see \autoref{prop:humphries}).
We then show an adversary against $\PI(\D_0, \D_1, n, \Train)$ whose advantage grows with $n$.

Let $\D_b = \mathrm{Bernoulli}(p_b)$ with $p_0 \neq p_1$ and
define $\Train(S) = \sum_{x \in S} x + \mathcal{N}(0, \sigma^2)$ where
$
  \sigma^2 = 2\ln(1.25/\delta) \varepsilon^{-1} .
$
Since the sum above has sensitivity 1 and \Train is the standard Gaussian mechanism, the training pipeline is ($\varepsilon, \delta$)-DP.

Note that for $S$ sampled from $\D_b$, the random variable $\Train(S)$ is distributed as
$
  \mathrm{Binomial}(n,p_b) + \mathcal{N}(0, \sigma^2) .
$
We can use Berry-Esséen theorem to approximate the binomial distribution with a normal distribution, so that approximately
\begin{align*}
  \Train(S) &\sim \mathcal{N}(n p_b, n p_b (1-p_b)) + \mathcal{N}(0, \sigma^2) \\
            &\sim \mathcal{N}(n p_b, n p_b (1-p_b) + \sigma^2)
\end{align*}
The approximation error is $O(\sqrt{n})$.
$\Train(S)/n$ is an unbiased estimator for $p_b$ with variance $p_b (1-p_b) + \nicefrac{\sigma^2}{n}$.
Since $\sigma$ does not depend on $n$, as $n$ grows the approximation error and the variance of the estimate decrease.
This allows a property inference adversary to distinguish between $\D_0$ and $\D_1$, particularly when $p_0$ and $p_1$ are far apart.
\end{proof}

\citet[Section 4.2]{Ateniese:2015} were the first to observe that differential privacy does not protect against property inference and provided a practical counterexample: a differentially private $k$-means network traffic classifier that nonetheless leaks the presence of traces from Google.com web traffic in their training dataset.
However, their argument remains informal, appealing to visual differences in the centroids of just two trained models.
\citet[Section V]{Suri:2023} give a more compelling example where property inference risk remains high on neural networks trained with DP-SGD.

\end{document}